\documentclass{article}
\usepackage{arxiv-style}
\usepackage[english]{babel}
\usepackage{amsthm}
\usepackage{mathtools}
\usepackage{xfrac} 
\usepackage[utf8]{inputenc} 
\usepackage[T1]{fontenc}    
\usepackage{hyperref}       
\usepackage{url}            
\usepackage{booktabs}    
\usepackage{amsfonts}      
\usepackage{nicefrac}      
\usepackage{microtype}    
\usepackage{lipsum}
\usepackage{fancyhdr}      
\usepackage{graphicx}       
\usepackage[ruled,vlined]{algorithm2e}
\usepackage{stackrel}

\usepackage{titlesec}
\usepackage{caption}
\usepackage{subcaption}

\usepackage{appendix}
\usepackage{csquotes}
\usepackage{amssymb}
\usepackage{sectsty}
\usepackage{accents}
\usepackage{comment}
\usepackage{mdframed}
\usepackage{tcolorbox}
\tcbuselibrary{listings, breakable}
\usepackage{array}
\usepackage{enumerate}
\usepackage[dvipsnames]{xcolor}
\usepackage[numbers,sort&compress]{natbib}

\hypersetup{
    colorlinks=true,    % use colored text instead of boxes
    linkcolor=blue,     % color for \ref, \pageref, etc.
    citecolor=magenta,  % color for \cite
    urlcolor=blue       % color for \url and \href
}

\numberwithin{equation}{section}

\newtheorem{theorem}{Theorem}[section]
\newtheorem{lemma}[theorem]{Lemma}
\newtheorem{proposition}[theorem]{Proposition}

\newtheorem{corollary}[theorem]{Corollary}

\theoremstyle{definition}
\newtheorem{definition}[theorem]{Definition}
\newtheorem{remark}[theorem]{Remark}

\newcommand{\R}{\mathbb{R}}

\newcommand{\be}{\begin{equation}}
\newcommand{\ee}{\end{equation}}

\makeatletter
\renewcommand{\fnum@figure}{Fig. \thefigure}
\makeatother

\def\EE{\mathbb{E}}

\def\PP{\mathbb{P}}
\newcommand{\cD}{\mathcal{D}}

\newcommand{\bra}[1]{\left( #1 \right)}

\newcommand{\cur}[1]{\left\{ #1 \right\}}

\newcommand{\cN}{\mathcal{N}}

\newcommand{\cX}{\mathcal{X}}

\renewcommand{\r}{\mathsf{r}}
\title{Beyond NNGP: Large Deviations and Feature Learning in Bayesian Neural Networks}

\author{
Katerina Papagiannouli\\ Dipartimento di Matematica\\ Università di Pisa\\ \texttt{aikaterini.papagiannouli@unipi.it}
\And Dario Trevisan\\Dipartimento di Matematica\\ Università di Pisa \\ \texttt{dario.trevisan@unipi.it}
\And
Giuseppe Pio Zito\\Dipartimento di Matematica\\ Università di Pisa \\ \texttt{g.zito7@studenti.unipi.it}
}
\begin{document}

\maketitle

\begin{abstract}%
    \sloppy
We study wide Bayesian neural networks focusing on the rare but statistically dominant fluctuations that govern posterior concentration, beyond Gaussian-process limits. Large-deviation theory provides explicit variational objectives-rate functions-on predictors, providing an emerging notion of complexity and feature learning directly at the functional level. We show that the posterior output rate function is obtained by a joint optimization over predictors and internal kernels, in contrast with fixed-kernel (NNGP) theory. Numerical experiments demonstrate that the resulting predictions accurately describe finite-width behavior for moderately sized networks, capturing non-Gaussian tails, posterior deformation, and data-dependent kernel selection effects. \end{abstract}

% \begin{keywords}%
% Large deviations; Bayesian learning; wide neural networks; Gaussian processes.\end{keywords}

\section{Introduction}

Bayesian neural networks offer a principled framework for uncertainty quantification and
regularization in modern machine learning.
In the overparameterized and infinite-width regime, their behavior is now well understood:
under mild assumptions, wide neural networks converge to Gaussian processes (NNGPs), and
training dynamics remain close to kernel methods described by neural tangent kernels
(NTKs).
These limits provide tractable models and sharp theoretical insights, but they also exhibit
a fundamental rigidity: the induced feature representation becomes fixed and independent of
data.
As a result, genuine feature learning disappears in the infinite-width limit, and Bayesian
inference reduces to kernel regression with a predetermined kernel.

This limitation raises a natural question: \emph{which objects govern learning beyond the
Gaussian-process scale?}
Most existing approaches address this question through training dynamics, mean-field limits,
or finite-width corrections.
In this work, we take a complementary perspective, by studying Bayesian inference and focusing on posterior concentration in the
large-width limit.

Our approach is based on large deviations theory.
While weak convergence results describe typical fluctuations of network outputs, large
deviation principles characterize the exponentially unlikely—but statistically dominant—
configurations that control posterior behavior.
At this scale, Bayesian inference becomes variational: posterior mass concentrates near
minimizers of explicit rate functions defined directly on predictors.
These rate functions act as emergent, architecture-dependent complexity penalties, modified
by conditioning on data.

The central contribution of this paper is to identify these large-deviation rate functions
as learning-relevant objects. In contrast to NNGP and NTK limits, the resulting variational formulation reveals a mechanism
for feature learning even in the infinite-width regime.
Rather than fixing a kernel in advance, posterior concentration selects data-dependent
kernels through a nested variational principle. Crucially, we argue that the theory is not merely qualitative but can be numerically implemented. We show that the resulting predictions accurately describe finite-width behavior, including
non-Gaussian tails, posterior modes, and systematic deviations from NNGP predictions.

\emph{Contributions.} In this work, we provide (i) a variational framework for Bayesian learning based on large-deviation rate functions,
operating directly at the level of predictors; (ii) theoretical explanation of feature learning in wide Bayesian neural networks as a
large-deviation phenomenon beyond Gaussian-process limits; (iii) numerical evidence showing that these rate functions are computable and accurately predict finite-width behavior.

\emph{Organization.} Section~\ref{sec:overview} introduces a general variational perspective on Bayesian learning via large
deviations. Section~\ref{sec:learning-interpretation} connects this viewpoint to learning-theoretic notions such as implicit
regularization, capacity control, and PAC-Bayes theory. Section~\ref{sec:nn} develops the main theoretical results, starting from Gaussian processes and extending to deep Gaussian neural networks. In Section~\ref{sec:numerics} presents numerical experiments validating the theory.
Technical proofs and additional numerical details are deferred to the appendix.

\section{A variational perspective on Bayesian learning}
\label{sec:overview}

We begin by fixing notation and introducing the probabilistic viewpoint that underlies the
paper.
We consider a supervised regression problem with a training dataset
\[
\mathcal D=\{(x_i,y_i)\}_{i\in D}\subset\mathbb R^{d_{\mathrm{in}}}\times\mathbb R^{d_{\mathrm{out}}},
\]
and a family of predictors \(h_\theta:\mathbb R^{d_{\mathrm{in}}}\to\mathbb R^{d_{\mathrm{out}}}\)
parametrized by \(\theta\).
Although learning algorithms are typically formulated in parameter space, prediction,
generalization, and uncertainty depend only on the induced functional behavior
\(x\mapsto h_\theta(x)\).
In modern overparameterized models—most notably deep neural networks—many distinct parameter
configurations induce essentially the same predictor, motivating an analysis that operates
directly at the level of functions rather than parameters.

Throughout the paper we focus on \emph{wide models}, indexed by a width parameter \(n\)
(e.g., the number of hidden units per layer).
For each width \(n\), the predictor \(h^n_\theta\) is random, either under a prior
distribution on parameters or under the corresponding Bayesian posterior obtained by
conditioning on data \(\mathcal D\).
To describe prediction in a finite-sample setting, we evaluate the predictor on a fixed
finite set of inputs \(\mathcal X\), containing both training and test points, and consider
the resulting random vector $H^n := \bigl(h^n_\theta(x)\bigr)_{x\in\mathcal X}$. This object takes values in a finite-dimensional function space and captures all quantities
relevant for prediction on \(\mathcal X\).

As the width \(n\) increases, the distribution of \(H^n\) usually exhibits strong concentration
phenomena: most of the probability mass accumulates near a small set of typical functional
behaviors.
Weak convergence results, such as
Gaussian-process limits, characterize typical fluctuations around a mean behavior.
However, such analysis alone discards information about rare but statistically relevant
events. Large deviations theory provides a complementary description, by quantifying the exponential
decay of probabilities of atypical configurations.
Informally, one expects
\[
\mathbb P(H^n\approx h)\;\approx\;\exp\{-n\,\mathcal I(h)\},
\]
where \(\mathcal I\) is a non-negative functional, called the \emph{rate function}.
At a conceptual level, it assigns a cost to each functional configuration $h$,
measuring how unlikely it is to arise from random parameters at large width.

Conditioning on data via a $n$-tempered likelihood $\exp\bra{-n \mathcal{L}(H^n)}$ modifies this picture in a particularly transparent way.
At the large-deviation scale, Bayesian updating corresponds to adding the empirical loss to
the prior rate function, yielding a posterior rate of the schematic form
\begin{equation}\label{eq:abstract-post}
\mathcal I_{\mathrm{post}}(h)
= 
\mathcal I_{\mathrm{prior}}(h)
+
\mathcal{L}(h)\quad  (+\text{constant})
\end{equation}
where \(\mathcal{L}\) denotes the loss function, and the  additive normalization constant ensures that $\inf \mathcal I_{\mathrm{post}} = 0$. As a consequence, posterior concentration  balances data fit and the probabilistic cost encoded by the prior rate.
Typical predictions correspond to minimizers of this effective objective, while uncertainty
can be studied through the geometry of the rate function around its minimizers.

Focusing on predictions at a test point \(x_{\mathrm{test}}\) , the variational analogue of marginalizing the
posterior distribution is given by the contraction
principle, which replaces integration by minimization. As a result, the maximum a posteriori (MAP) prediction reads $h^*(x_{\mathrm{test}})$, where
\begin{equation}\label{eq:prediction}
h^*
\;\in\;
\arg\min_{h}
\;
 \mathcal I_{\mathrm{post}}(h)
\end{equation}
%This formula is the large-deviation analogue of Bayesian marginalization and makes explicit the variational nature of prediction in wide Bayesian models.

This perspective identifies a concrete object—the functional rate function—that governs both
prior and posterior behavior in wide Bayesian models and operates directly at the level of
predictors.
In the remainder of this work, we specialize this framework to neural networks with Gaussian
weights, where these rate functions can be characterized explicitly and computed
numerically.
%We show that they retain information beyond classical Gaussian approximations and accurately describe finite-width behavior.

\section{Learning-theoretic interpretation and connections}
\label{sec:learning-interpretation}

%This section interprets the variational objects introduced above from a learning-theoretic perspective and situates them within existing lines of work.
We argued above that posterior concentration in wide Bayesian models
is governed by large-deviation rate functions defined directly on predictors.
Here we explain how rate functions can be viewed as emergent notions of complexity and
capacity control, providing a  bridge between probabilistic modeling
and learning-theoretic reasoning.

\emph{Emergent complexity and implicit regularization.}
A central observation is that the prior rate function $\mathcal{I}_{\mathrm{prior}}$ acts as
an intrinsic penalty on functional behavior.
Unlike classical regularization, which is imposed explicitly through design choices, this
penalty emerges from concentration of measure in parameter space as the width grows.
At the level of macroscopic network behavior, the rate function assigns a cost to each
predictor \(h\), quantifying how statistically accessible that behavior is under random
(prior) parameterization.
%Functional configurations that require a high degree of coordination among microscopic parameters incur a large rate cost and are therefore exponentially unlikely in the wide-network regime.
This interpretation is consistent with a broad body of work on implicit regularization in
overparameterized models, where generalization is governed by biases induced by the model
and training procedure rather than by explicit constraints
\cite{zhang2017understanding,belkin2019reconciling,chizat2020implicit}.

\emph{Capacity control beyond hypothesis classes.}
From a learning-theoretic viewpoint, the rate function formalism induces a form of capacity
control that differs from uniform notions such as VC dimension or global norm constraints.
Rather than defining a hard hypothesis class, the rate function induces a non-uniform
weighting over functional behaviors.
%Predictors are not simply allowed or excluded; instead, they are weighted according to their statistical accessibility under the model.
This perspective is particularly natural again in highly overparameterized settings, where the
dimension of parameter space alone provides little information about effective complexity.
Related analyses have emphasized geometry- and norm-based notions of capacity over parameter
counting, highlighting the limitations of classical uniform bounds in modern regimes
\cite{bartlett2017spectrally,bartlett2020benign,dziugaite2017computing,datres2024two}.
%In the wide-network limit, the geometry of the rate function plays an analogous role by determining which predictors are statistically typical and which are exponentially suppressed.

\emph{Variational principles and posterior concentration.}
%Conditioning on data further clarifies the connection between rate functions and learning.
At the large-deviation scale, Bayesian updating corresponds to a change of measure that adds
the empirical loss to the prior rate function.
%As a result, typical posterior predictions minimize a variational objective of the form \eqref{eq:abstract-post}. 
This structure clearly mirrors regularized empirical risk minimization, with a crucial difference:
the regularizer is not postulated a priori, but is induced by the stochastic architecture and
the probabilistic wide-network limit.
From this viewpoint, implicit regularization in Bayesian neural networks can be understood as
posterior mass concentrating near minimizers of an emergent, complexity-regularized objective.
Related variational perspectives on learning and implicit bias have appeared in mean-field
and information-theoretic analyses of neural networks
\cite{tishby2015deep,chizat2018global,mei2018mean}.

\emph{Connections to PAC-Bayes theory.}
The variational structure induced by large-deviation rate functions also connects, at least in a
conceptual level, to PAC-Bayes formulations.
In PAC-Bayes theory, generalization guarantees are expressed through a trade-off between
empirical risk and a complexity term given by the Kullback--Leibler divergence between
posterior and prior distributions
\cite{mcallester1999pac,catoni2007pac,casado2024pac}.
Large-deviation principles suggest that, in the wide-network regime, this complexity term
effectively concentrates on a functional rate evaluated at the typical predictor induced by
the posterior.
In this sense, the rate function may be viewed as a sort of KL-type complexity penalty acting on
macroscopic observables rather than on distributions over parameters.
This connection is currently heuristic rather than quantitative, and is intended to clarify
structure rather than to yield explicit bounds.

\emph{Related perspectives in neural networks.}
Infinite-width limits of neural networks leading to Gaussian-process behavior are by now
well understood and provide a tractable description of typical fluctuations in wide models
\cite{neal1996bayesian,lee2018deep,matthews2018gaussian}.
Neural tangent kernel limits further formalize this picture by showing that, in certain
regimes, training dynamics remain close to a fixed kernel model
\cite{jacot2018neural}.
More recent work has explored infinite-width limits that move beyond fixed kernels and admit
data-dependent kernels capturing aspects of feature learning
\cite{yang2021tensor,lauditi2025adaptive,pacelli2023statistical,chizat2024infinite}. Other works have shown that shallow—and more recently deep—neural networks are more
accurately characterized by suitable reproducing kernel Banach spaces, whose norms promote
structured sparsity and adaptivity in learned representations
\cite{zhang2009reproducing,bartolucci2023understanding,wang2024sparse}.
While these approaches characterize the ambient function space associated with a neural
architecture, our focus is different: we study the probabilistic geometry governing which
functions are statistically dominant in wide Bayesian models. This probabilistic viewpoint complements functional-analytic characterizations and yields concrete predictive and numerical consequences.

% ============================================================
\section{Gaussian Processes and Gaussian Neural Networks}
\label{sec:nn}
% ============================================================

We specialize the general variational perspective  sketched in Section~\ref{sec:overview} to two models. First, as a warm-up, we consider a fixed Gaussian-process (GP) prior, where the rate function reduces to an RKHS norm. Then, we turn to wide Gaussian neural networks, where we prove that rate functions lead to a joint variational problem over predictors and kernels, providing a tractable mechanism for feature learning. All proofs are deferred to the appendix.

% ------------------------------------------------------------
\paragraph{GP case -- Fixed Kernels.}
% ------------------------------------------------------------
We fix a finite input set $\cX =\cur{x^{(1)}, \ldots, x^{(m)} }\subseteq \R^{d_{\mathrm{in}}}$ containing both training and test points, and let
$H=(H(x))_{x\in\cX} \sim \mathcal{N}(0, \kappa)$ be a centered Gaussian vector with covariance matrix $\kappa $ (for simplicity we present the scalar-output case $d_{\mathrm{in}} =1$). \emph{Prior rate.} The large-deviation scaling is obtained by introducing $H^n := \tfrac{1}{\sqrt{n}} H$, so that the Gaussian vector induces (by Cram\'er theorem, see Appendix~\ref{app:ld}) the prior rate
\begin{equation}\label{eq:gp-prior-rate}
I_{\mathrm{prior}}^{\kappa}(h)=\frac12\|h\|_{\kappa}^2
\;=\;\frac12\, h^\top \kappa^{+} h \in [0, \infty]
\qquad h\in\R^{m},
\end{equation}
where $\kappa^{+}$ denotes the Moore--Penrose pseudoinverse, if $h$ orthogonal to $\operatorname{Ker}(\kappa)$, otherwise we set $I_{\mathrm{prior}}^{\kappa}(h) = \infty$. Equivalently,  $\|\cdot\|_{\kappa}$ is the RKHS (extended) norm on $\cX$ associated with the kernel $\kappa$. \emph{Posterior rate under quadratic loss.} Write $y_D=(y_i)_{i\in D}$ for the train set outputs, and consider the usual quadratic loss. Then, the posterior rate formula \eqref{eq:abstract-post} reads 
\begin{equation}\label{eq:gp-post-rate}
I_{\mathrm{post}}^{\kappa}(h)
=
\frac{1}{2}\sum_{i\in D}(h(x_i)-y_i)^2
+\frac12\|h\|_{\kappa}^2
\quad (+\text{constant}).
\end{equation}
\emph{MAP prediction.} For a test input $x_{\mathrm{test}}\in\cX$, the prediction \eqref{eq:prediction} for $h^*$ becomes explicitly
\begin{equation}\label{eq:gp-map}
h^* \in\arg\min_{h\in\R} \left\{
\frac{1}{2}\sum_{i\in D}(h(x_i)-y_i)^2
+\frac12\|h\|_{\kappa}^2
\right\}.
\end{equation}
%\paragraph{Closed form and the ``min--min exchange''.}
Of course, these formulas are equivalent to the usual GP/RKHS approach, hence an explicit formula for $h^*$ is available. Denote by $\kappa_{DD}$ the restriction of $K$ to training indices, and $\kappa_{xD}$ the row vector $(\kappa(x,x_i))_{i\in D}$.
Then, the classical GP posterior is $H(x_{\mathrm{test}})\sim \cN(m(x_{\mathrm{test}}),\sigma^2(x_{\mathrm{test}}))$,
\begin{equation}\label{eq:gp-posterior-mean-var}
m(x)=\kappa_{xD}(\kappa_{DD}+ \operatorname{Id})^{-1}y_D,\qquad
\sigma^2(x)=\kappa(x,x)-\kappa_{xD}(\kappa_{DD}+ \operatorname{Id})^{-1}\kappa_{Dx}.
\end{equation}
and therefore since MAP and mean coincide for Gaussians, we obtain $h^*(x_{\mathrm{test}})=m(x_{\mathrm{test}})$. 
% \begin{equation}\label{eq:gp-pred-rate-1d}
% I_{\mathrm{post}}^{\kappa}(y_D, y)=\frac{(y-m(x_{\mathrm{test}}))^2}{2\,\sigma^2(x_{\mathrm{test}})}\quad (+\text{constant}),
% \end{equation}
% so the MAP prediction equals the mean: $y^*(x_{\mathrm{test}})=m(x_{\mathrm{test}})$.  This solution shows explicit dependence upon the training data, and also the fixed kernel $\kappa$.

% Finally, we notice that we can eliminate $h$ from \eqref{eq:gp-post-rate} and obtain an objective depending only on the kernel.
% Specifically, integrating/minimizing over $h$ yields (up to constants) the standard kernel ridge regression energy
% \begin{equation}\label{eq:gp-minmin-exchange}
% \min_{h}\; I_{\mathrm{post}}^{K}(h)
% \;=\;
% \frac12\, y_D^\top (\kappa_{DD}+ \operatorname{Id})^{-1} y_D
% \quad (+\text{constant}),
% \end{equation}
% making explicit that, for a fixed GP prior, learning acts in a \emph{fixed} feature geometry (fixed $K$).

% ------------------------------------------------------------
\paragraph{Wide Gaussian Neural Networks -- Kernel Selection.}
% ------------------------------------------------------------
We now move to fully connected networks with Gaussian weights and biases.
Let $L\ge2$, $d_0=d_{\mathrm{in}}$, $d_L=d_{\mathrm{out}}(=1)$, and hidden layers widths $d_1=\cdots=d_{L-1}=n$ (all equal for simplicity).
For an input $x\in\R^{d_0}$ define the forward recursion
\begin{equation}\label{eq:nn-recursion-short}
h^{(0)}(x)=x,\qquad
h^{(\ell)}(x)=W^{(\ell)}\sigma^{(\ell-1)}(h^{(\ell-1)}(x))+b^{(\ell)},\quad \ell=1,\dots,L,
\end{equation}
with (Lipschitz) activation functions acting componentwise, and independent Gaussian parameters according to LeCun scaling $W^{(\ell)}_{ij}\sim\cN\!\left(0,1/{d_{\ell-1}}\right)$, $b^{(\ell)}_{j}\sim\cN(0,1)$, independently over indices and layers. 
Fix a input (train/test) set $\cX$, and define the empirical layerwise kernels by $K_n^{(\ell)}(\cX)$ via 
\[
K_n^{(\ell)}(x, x')
:=\frac1n\sum_{j=1}^{n}\sigma^{(\ell)}(h^{(\ell)}_j(x))^{\top} \sigma^{(\ell)}(h^{(\ell)}_j(x')),
\qquad x, x' \in \cX, \quad \ell=1,\dots,L-1,
\]
and $K_n^{(0)}(\cX)=\tfrac1{d_0}\, \cX \cX^\top$ for the input layer. In the wide limit $n\to\infty$, $K_n^{(\ell)}$ converges to a deterministic limit $\kappa_0^{(\ell)}$, the NNGP kernel,  that can be recursively computed (\cite{matthews2018gaussian}). % Thus, the output process has a fixed-kernel GP description at the level of weak convergence. 

% ------------------------------------------------------------
\emph{Prior rate for kernels.}
% ------------------------------------------------------------
The point of the LDP approach is to describe the exponential cost of rare deviations away from $\kappa_0^{(\ell)}$,  and to understand how the posterior selects atypical kernels. The following reslust was first obtained in \cite{macci2024large}, see also \cite{andreis2025ldp}.

\begin{proposition}[Prior LDP for layerwise kernels; NNGP as unique minimizer]
\label{thm:prior-kernel-ldp}
Fix $\cX$ and a depth $L$, and consider the Gaussian network \eqref{eq:nn-recursion-short} with hidden widths $n$.
For each $\ell=1,\dots,L-1$, the sequence $(K_n^{(\ell)}(\cX))_{n\ge1}$ satisfies a large-deviation principle with rate function 
\begin{equation}\label{eq:kernel-ldp-recursion}
I^{(\ell)}(\kappa^{(\ell)})
=
\inf_{\kappa^{(\ell-1)} }
\Big\{J^{\sigma^{(\ell)}}(\kappa^{(\ell)} | \kappa^{(\ell-1)}) + 
I^{(\ell-1)}(\kappa^{(\ell-1)})
\Big\},
\end{equation}
where $J^{\sigma^{(\ell)}}(\cdot| \kappa^{(\ell-1)}) $ is a layer cost -- given as a Legendre--Fenchel transform of the conditional $\log$-MGF of $K^{(\ell)}$, given $K^{(\ell-1)} =  \kappa^{(\ell-1)}$. Moreover, the NNGP kernel $\kappa_0^{(\ell)}$ minimizes $I^{(\ell)}$.
\end{proposition}

Equation \eqref{eq:kernel-ldp-recursion} is the large-deviation analogue of the forward pass: typical behavior sets $\kappa^{(\ell)}=\kappa_0^{(\ell)}$, whereas atypical kernels incur a cost accumulated layer-by-layer.

% ------------------------------------------------------------
\emph{Prior rate for outputs.}
\label{subsec:prior-output-ldp}
% ------------------------------------------------------------
Let $H_n:=\tfrac{1}{\sqrt{n}}(h_\theta(x))_{x\in\cX}$ be the rescaled output vector. Conditionally on $K_n^{(L-1)}= \kappa^{(L-1)}$, the output is Gaussian with covariance $\kappa^{(L-1)}$, hence deviations of $H_n$ have a conditionally quadratic RKHS-type cost. Such intuition is made rigorous in \cite{macci2024large}.

\begin{theorem}[Prior LDP for outputs]
 \label{thm:prior-output-ldp}
Under the assumptions of Proposition~\ref{thm:prior-kernel-ldp}, the sequence $(H_n)_{n\ge1}$ satisfies an LDP with speed $n$ and rate function given by  ($\|\cdot\|_{\kappa}$ being the RKHS norm induced by $\kappa$)

\begin{equation}\label{eq:output-rate-min-kappa}
I_{\mathrm{prior}}^{(L)}(h)
=
\inf_{\kappa }
\left\{
\frac12\|h\|_{\kappa}^{2}
+
I^{(L-1)}(\kappa)
\right\}.
\end{equation}
\end{theorem}

Compared to the GP case \eqref{eq:gp-prior-rate}, the difference is the extra minimization over $\kappa$.
Thus, even under the prior, outputs deviations $H_n \approx h$ will select a kernel $\kappa^{(\ell), *} = \kappa^{(\ell),*}(h)$ for each hidden  layer.

% ------------------------------------------------------------
\emph{Posterior rate.}
% ------------------------------------------------------------
Let $\cD=\{(x_i,y_i)\}_{i\in D}$ and consider as in the GP case the quadratic loss $\mathcal L(h)=\tfrac{1}{2  }\sum_{i\in D}\|h(x_i)-y_i\|_2^2$. We consider the posterior law via the tempered change of measure, i.e. proportional to $\exp\{-n\,\mathcal L( H_n)\}$ (or equivalently, a width-dependent inverse temperature). In this situation, we argue that the scaled network outputs $H_n$ follow the posterior rate obtained from \eqref{eq:abstract-post}. The proof is a consequence of Varadhan's lemma and is provided in Appendix~\ref{app:proof}. Let us point out that a similar argument appears in  \cite{andreis2025ldp}, for losses that depend uniquely on empirical kernels, not for scaled network outputs.

\begin{theorem}[Posterior LDP for rescaled outputs under quadratic loss]
\label{thm:posterior-output-ldp}
Under the assumptions of Theorem~\ref{thm:prior-output-ldp} and the tempered posterior described above, $(H_n)_{n\ge1}$ satisfies an LDP, with speed $n$, and rate function
\begin{equation}\label{eq:posterior-rate-min-kappa}
I_{\mathrm{post}}^{(L)}(h)
=
\inf_{\kappa\in\mathbb S_+^{m}}
\left\{
\frac{1}{2}\sum_{i\in D}\|h(x_i)-y_i\|_2^2 + \frac12\|h\|_{\kappa}^{2}
+
I^{(L-1)}(\kappa) 
\right\}
\quad (+\text{constant}).
\end{equation}
\end{theorem}

\emph{MAP prediction and kernel selection.}
For a test input $x_{\mathrm{test}}\in\cX$, the prediction rate is obtained by contraction principle and the MAP prediction by \eqref{eq:prediction}, which specialized to \eqref{eq:posterior-rate-min-kappa} makes explicit a nested optimization structure:
prediction jointly selects a predictor $h^{\star}(x_{\mathrm{test}})$ but also  internal kernels $\kappa^{(\ell),*}$, depending on data, which we can be intepreted as emergent feature representation. %We argue that that generically such kernels will differ from the NNGP ones.

%Let $\kappa_0^{(L-1)}$ denote the unique minimizer of $I^{(L-1)}$ (the NNGP kernel at the last hidden layer). The next result formalizes that posterior concentration typically selects a different kernel.

\begin{corollary}[Posterior-optimal kernel differs from the NNGP kernel]
\label{thm:kernel-separation}
Assume Theorem~\ref{thm:posterior-output-ldp}. The posterior kernel rate, obtained by setting $h=h^*$ in the rhs cost in \eqref{eq:posterior-rate-min-kappa}, reads
\begin{equation}\label{eq:phi-kappa}
I^{(L-1)}(\kappa | \mathcal{D}, x_{\mathrm{test}} )
=
I^{(L-1)}(\kappa  )
+\frac12 
y^\top_{D} (\kappa_{DD}+\operatorname{Id})^{-1}y_D
\end{equation}
% where $Y$ collects the training outputs (for $d_{\mathrm{out}}=1$, $\mathrm{Tr}(Y^\top A^{-1}Y)=y_D^\top A^{-1}y_D$).
If $y_{D} \neq 0$, the posterior-optimal kernel $\kappa^\star\in\arg\min_{\kappa}I^{(L-1)}(\kappa | \mathcal{D},  x_{\mathrm{test}})$ satisfies $\kappa^\star\neq \kappa_0$.
\end{corollary}

The posterior rate \eqref{eq:phi-kappa} is obtained by noticing that in the nested optimization leading to $y^*(x_{\mathrm{test}})$ by exchanging the order, i.e.\ minimizing first over $y$, for fixed $\kappa$ and $h(x_D) = y_D$, we reduce the problem to the fixed kernel case, hence $y^*(x_{\mathrm test})$ is the mean prediction \eqref{eq:gp-posterior-mean-var} and the loss contribution is also explicit. This can be intepreted as follows: the posterior rate selects a data-dependent kernel by trading off the kernel rarity cost $I^{(L-1)}(\kappa)$ against the standard kernel-regression fit term: by this mechanism feature learning appears at the LDP scale, beyond the rigidity of NNGP limits at the weak-convergence scale.

\section{Numerical Experiments}
\label{sec:numerics}

We demonstrate the theoretical predictions from the large-deviation framework and assess their relevance for finite-width neural networks, validating that the resulting rate functions are computable, interpretable, and predictive beyond Gaussian approximations.

Numerical experiments are conducted in Python using JAX, with the complete codebase shared on GitHub \cite{ldnn2026}. While our implementation may be extended to support various architectures, we focus on a shallow case (one hidden layer) for clarity, capturing key phenomena from Sections \ref{sec:overview}-\ref{sec:nn}. Computations are lightweight and run on standard CPUs, though caution is advised for numerical instabilities. Further diagnostics and stabilization strategies are discussed in Appendix \ref{app:numerics}.

\subsection*{Experiment 01: prior and posterior rate functions, LDP-MAP prediction}

\paragraph{01A: Prior rate function.} We compute the prior output large-deviation rate function $I_{\mathrm{prior}}$ on a one-dimensional output grid to illustrate its  geometry. We consider a single hidden layer network in the wide regime, and two activation functions:
$\sigma(z)=\mathrm{ReLU}(z)$ and $\sigma(z)=\tanh(z)$. Input and output dimensions are $d_{\mathrm{in}}=d_{\mathrm{out}}=1$. The test input is fixed at $x_{\mathrm{test}}=3$. We compute rate function $I_{\mathrm{prior}}(y)$ evaluated on a uniform grid $y\in[-1,4]$. 

The results are plotted in Figure~\ref{fig:exp01A_prior_rate}. While typical fluctuations are Gaussian at the NNGP scale, the rate function geometry depends strongly on the activation function (both in scale and in shape).
The ReLU activation, being unbounded, leads to a rate function with
markedly non-quadratic growth (see also \ref{app:explicit}), in contrast with the smoother behavior observed for bounded
activations.

\begin{figure}[t]
\centering
\begin{minipage}{0.48\linewidth}
\centering
\includegraphics[width=\linewidth,  keepaspectratio]{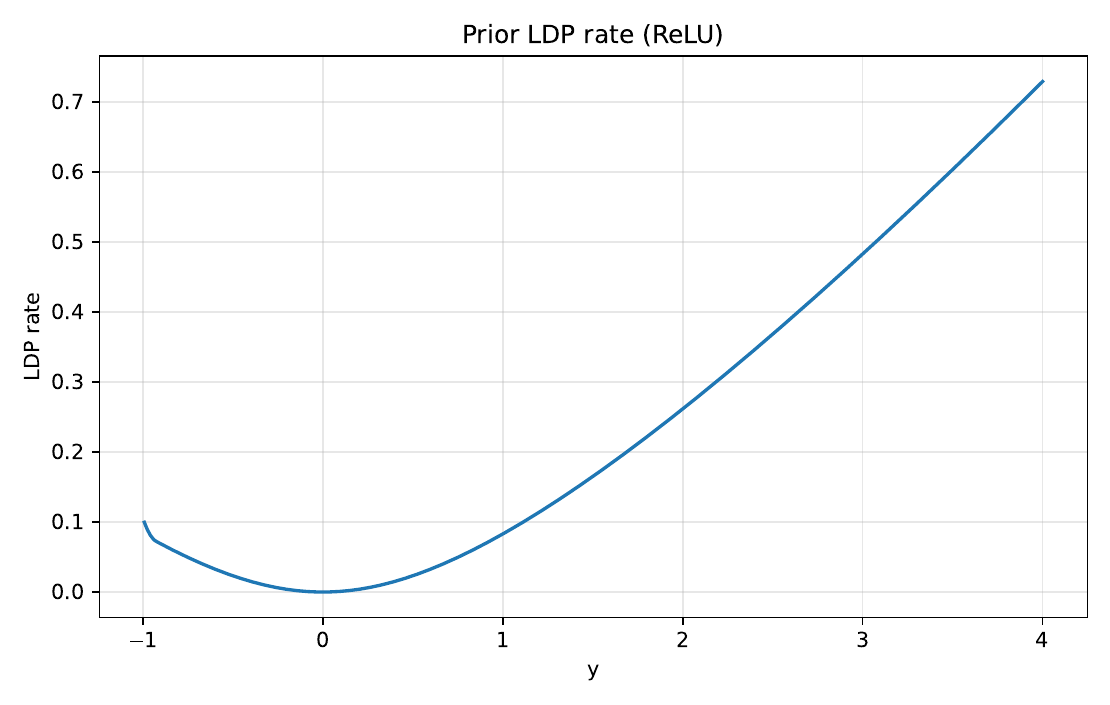}
\end{minipage}\hfill
\begin{minipage}{0.48\linewidth}
\centering
\includegraphics[width=\linewidth, keepaspectratio]{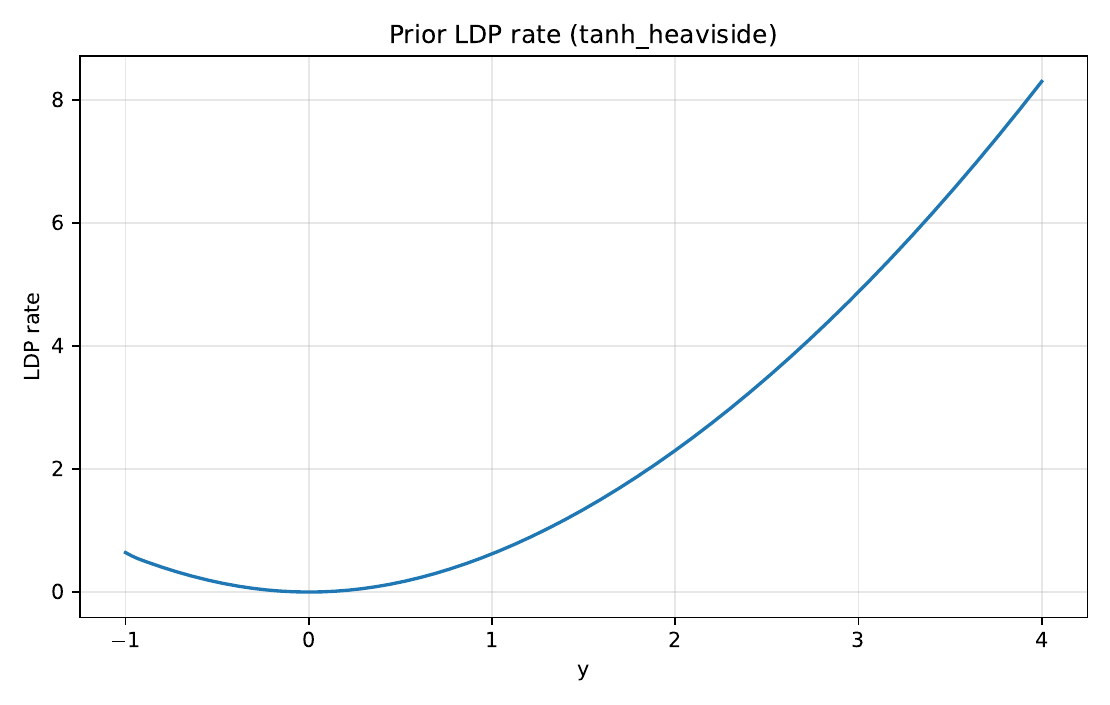}
\end{minipage}
\caption{\textbf{Prior output large-deviation rate functions.}
Prior rate  $I_{\mathrm{prior}}(y)$ as a function of the output $y$ for a wide
Gaussian neural network with ReLU (left) and $\tanh$ (right) activation.}
\label{fig:exp01A_prior_rate}
\end{figure}

\paragraph{01B: Posterior rate function.}
The goal is to compute the posterior output large-deviation rate function
$I_{\mathrm{post}}$ and illustrate how conditioning on data deforms the prior
large-deviation geometry.
We consider the same fully connected neural network with a single hidden layer and Gaussian
weights as in Experiment~01A. Again we plot both the case of ReLU activation function and $\tanh$, for comparison and consistency.
The training set consists of samples from a one-dimensional Heaviside target function,
evaluated at the inputs $x_{\mathrm{train}}\in\{-3,-2,-1,0,1,2\}$.
Input and output dimensions are $d_{\mathrm{in}}=d_{\mathrm{out}}=1$, and the test input is
fixed at $x_{\mathrm{test}}=3$. We are interested in the posterior output rate function $I_{\mathrm{post}}(y)$ induced by a quadratic loss, evaluated on the grid 
$y\in[-1,4]$ and compared with the corresponding prior rate function.

The results are shown in Figure~\ref{fig:exp01B_posterior_rate}. Conditioning on data induces a nontrivial deformation of the large-deviation landscape: the posterior rate function is shifted and dilated relative to the prior,
taking into account the contribution of the loss function.

\begin{figure}[t]
\centering
\begin{minipage}{0.48\linewidth}
\centering
\includegraphics[width=\linewidth,  keepaspectratio]{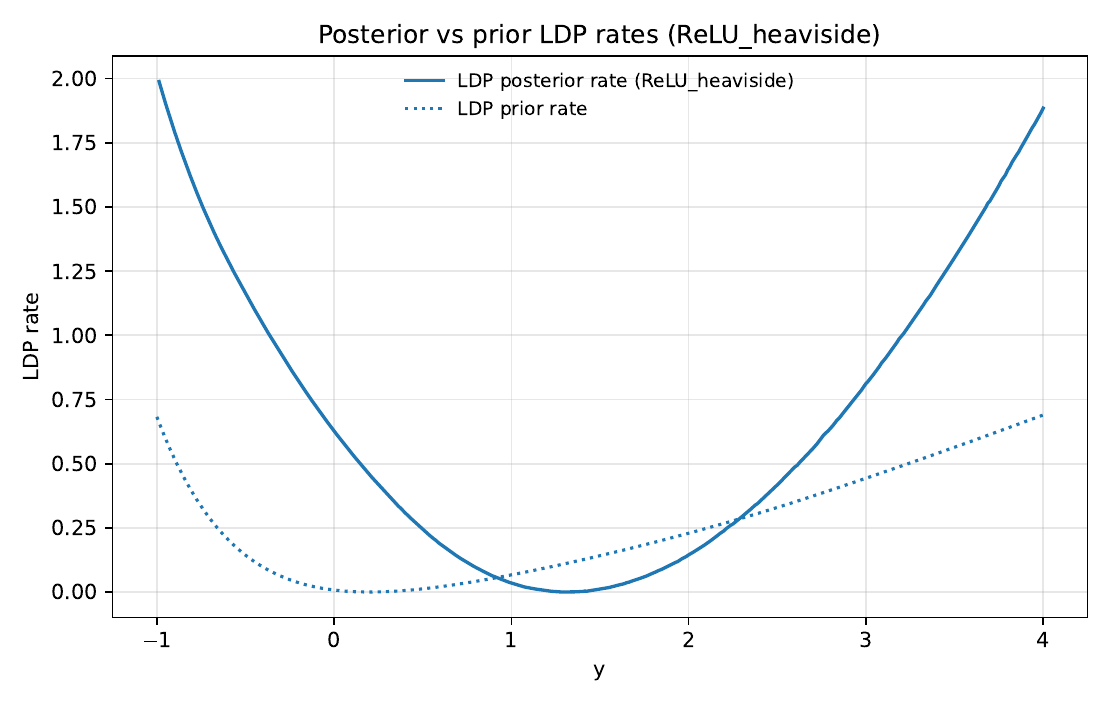}
\end{minipage}\hfill
\begin{minipage}{0.48\linewidth}
\centering
\includegraphics[width=\linewidth, keepaspectratio]{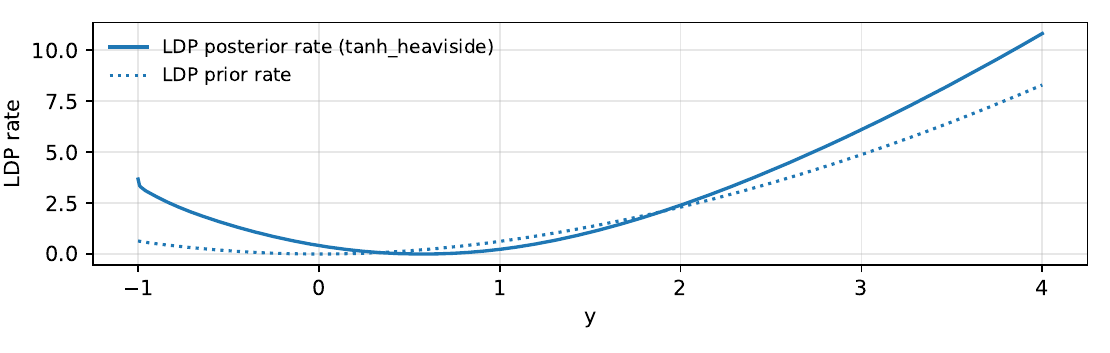}
\end{minipage}
\caption{\textbf{Posterior deformation of the output large-deviation rate function.}
Prior and posterior rate functions $I_{\mathrm{prior}}(y)$ and $I_{\mathrm{post}}(y)$ at a
fixed test input $x_{\mathrm{test}}=3$, for a wide Gaussian neural network (left, ReLU activation; right $\tanh$ activation) trained on a Heaviside target}
\label{fig:exp01B_posterior_rate}
\end{figure}

%----------------
\paragraph{01C: MAP prediction (mode) as a function of the input.}
We compute the large-deviation MAP prediction induced by the posterior rate function,
defined for each test input $x_{\mathrm{test}}$ the prediction $y^*(x_{\mathrm{test}})$ as in \eqref{eq:prediction}. %\emph{Architecture and training data.}
We consider the same shallow Gaussian neural network and training set as in
Experiment~01B, with ReLU and $\tanh$ activation functions.
%\emph{Inputs and outputs.}
Input and output dimensions are $d_{\mathrm{in}}=d_{\mathrm{out}}=1$.
The test input $x_{\mathrm{test}}$ is varied over a uniform grid
$x_{\mathrm{test}}\in[-4,4]$.
For each value of $x_{\mathrm{test}}$, the corresponding MAP prediction $y^\ast(x_{\mathrm{test}})$
is obtained by minimizing the posterior rate function with respect to the output variable.

he resulting prediction curves are shown in Figure~\ref{fig:exp01C_mode_curve}.
Although the posterior MAP predictor does not interpolate the training data exactly—an
expected effect given the quadratic loss and the limited size of the training set—it
captures the qualitative monotone transition of the target function from $0$ to $1$.
Differences induced by the activation function are clearly visible:
the ReLU activation produces a steeper transition and does not saturate at $1$, while the
$\tanh$ activation leads to a smoother profile with saturation effects.

\begin{figure}[t]
\centering
\begin{minipage}{0.48\linewidth}
\centering
\includegraphics[width=\linewidth]{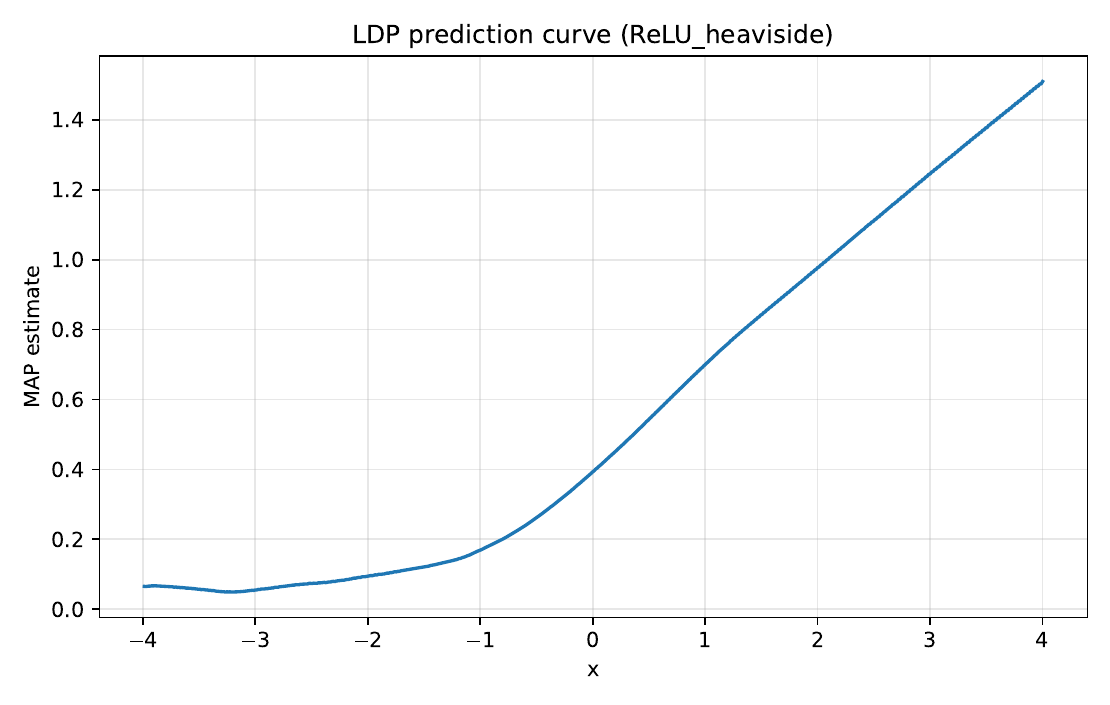}
\end{minipage}\hfill
\begin{minipage}{0.48\linewidth}
\centering
\includegraphics[width=\linewidth]{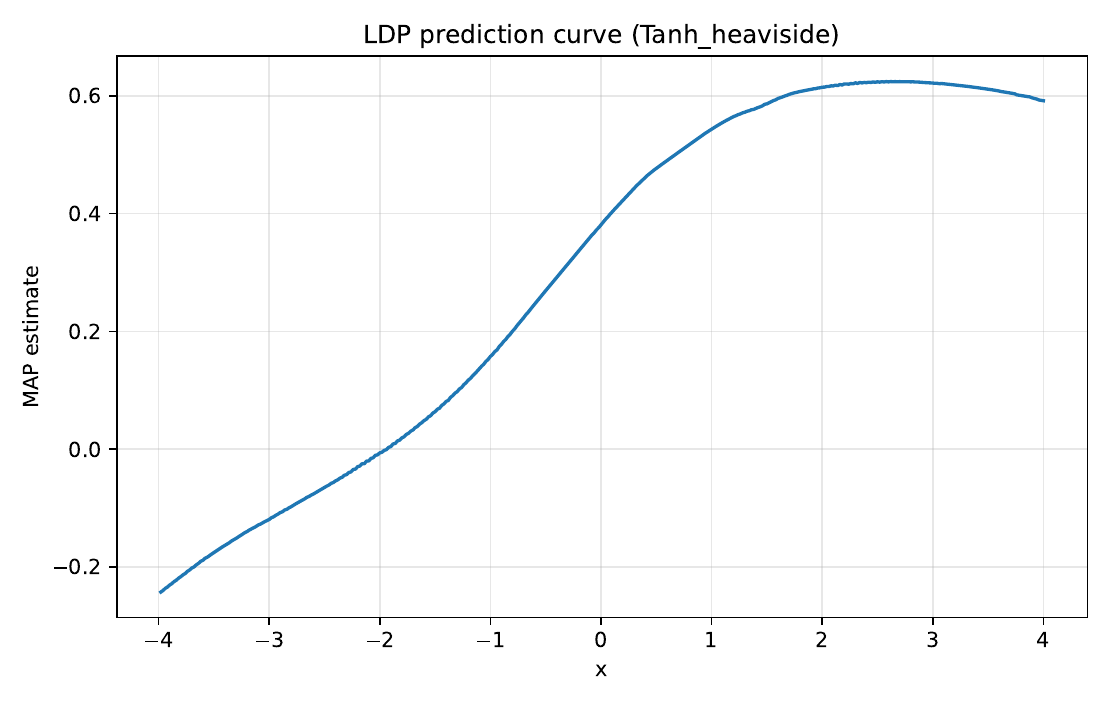}
\end{minipage}
\caption{\textbf{Posterior MAP prediction curves.}
Large-deviation MAP prediction $y^\ast(x_{\mathrm{test}})$ as a function of the test input
$x_{\mathrm{test}}$, for a wide Gaussian neural network trained on a Heaviside target.
Left: ReLU activation. Right: $\tanh$ activation.}
\label{fig:exp01C_mode_curve}
\end{figure}

% --------------------------

\subsection*{Experiment 02: LDP versus NNGP (Gaussian) predictions}
\label{subsec:exp02}

This experiment compares LDP predictions with the fixed-kernel NNGP approximation,
highlighting the regimes where Gaussian limits correctly describe typical behavior but
fail at the large-deviation scale.
\paragraph{02A: Prior rate --- LDP versus NNGP quadratic rate}

We compare the prior large-deviation rate function with the quadratic rate induced by the
Neural Network Gaussian Process (NNGP) limit. We consider the same architecture and input configuration as in Experiment~01A, restricting to the ReLU activation for simplicity. Input and output dimensions are $d_{\mathrm{in}}=d_{\mathrm{out}}=1$, and the test input is
fixed at $x_{\mathrm{test}}=3$.
We compute, on the same output grid $y\in[-1,4]$, the prior large-deviation rate function $I_{\mathrm{prior}}^{\mathrm{LDP}}(y)$ and the quadratic rate $I_{\mathrm{prior}}^{\mathrm{NNGP}}(y)
=
\frac12 \|y\|^2_{\kappa_0},
$ where $\kappa_0$ denotes the NNGP kernel. In addition, for each output value $y$ we compute the relative kernel gap between the kernel
$\kappa^\star(y)\in\arg\min_\kappa I_{\mathrm{prior}}^{\mathrm{LDP}}(y)$ selected by the LDP
variational problem and the NNGP kernel $\kappa_0$, defined as $\|\kappa^\star(y)-\kappa_0\|_{\mathrm{op}}/{\|\kappa_0\|_{\mathrm{op}}}$.

The results are shown in Figure~\ref{fig:exp02A_prior_ldp_vs_nngp}.
The left panel displays the prior LDP rate function together with the quadratic NNGP rate,
while the right panel shows the corresponding relative kernel gap as a function of the
output value. As expected, the two rates coincide around $y=0$ (up to small numerical imprecisions), which
is the unique minimizer of the prior rate function. Away from this typical regime, the LDP rate departs from the quadratic approximation and the
kernel gap becomes non-zero, indicating that rare output fluctuations are governed by
non-Gaussian kernel configurations.

\begin{figure}[t]
\centering
\begin{minipage}{0.48\linewidth}
\centering
\includegraphics[width=\linewidth]{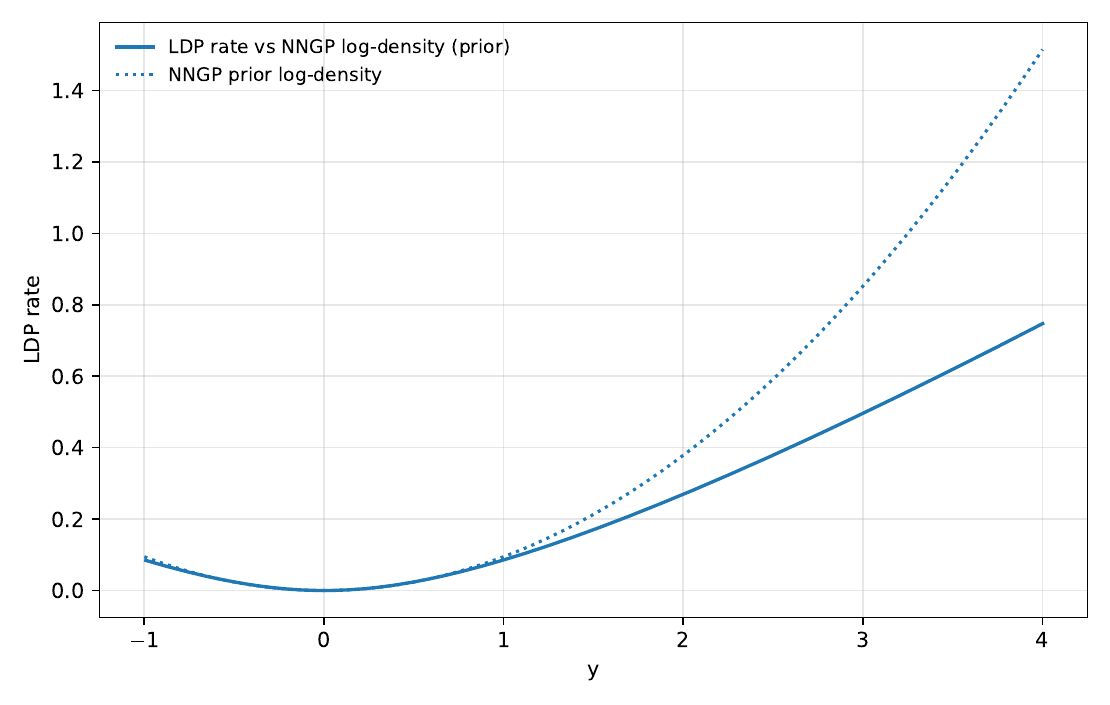}
\end{minipage}\hfill
\begin{minipage}{0.48\linewidth}
\centering
\includegraphics[width=\linewidth]{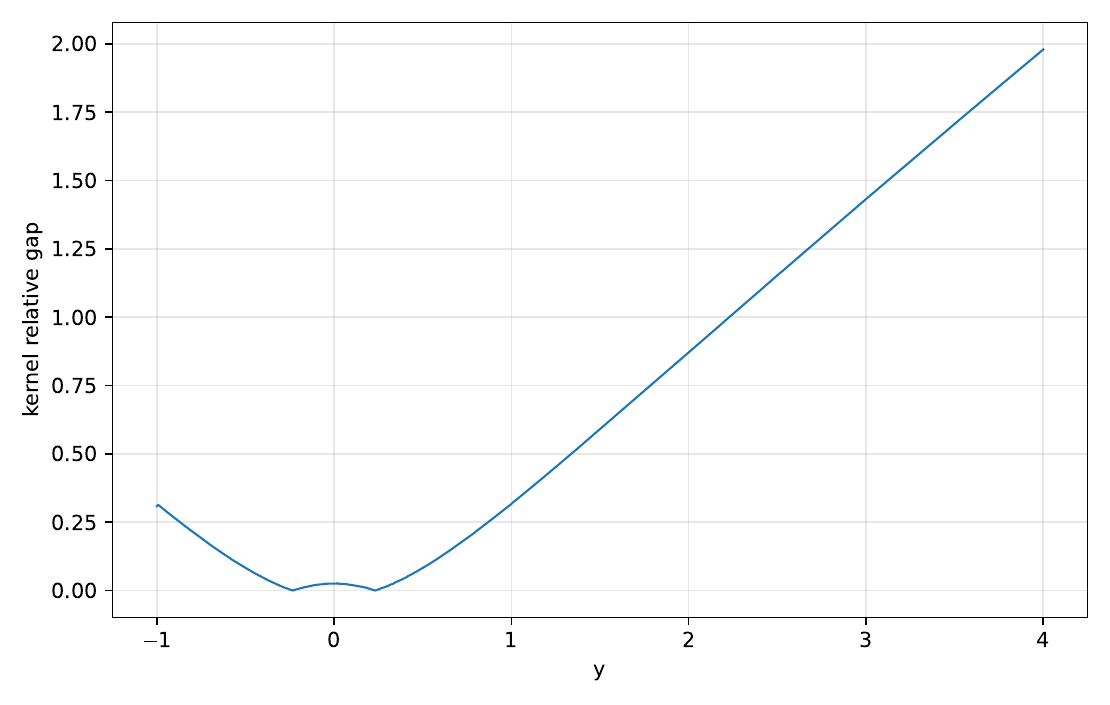}
\end{minipage}
\caption{\textbf{Prior LDP versus NNGP.}
Left: prior large-deviation rate function compared with the quadratic rate induced by the
NNGP kernel.
Right: relative operator-norm gap between the kernel selected by the LDP variational
problem and the NNGP kernel.}
\label{fig:exp02A_prior_ldp_vs_nngp}
\end{figure}

\paragraph*{02B: Posterior rate --- LDP versus NNGP posterior}

We compare the posterior large-deviation rate function with the posterior rate induced by
Gaussian-process regression using the fixed NNGP kernel. The goal is to assess that Bayesian conditioning does not attenuate non-Gaussian behavior at the large-deviation scale. We consider the same shallow Gaussian neural network architecture and training set as in Experiment~01B, restricting again to the ReLU activation for simplicity.
%Input and output dimensions are $d_{\mathrm{in}}=d_{\mathrm{out}}=1$, and the test input is fixed at $x_{\mathrm{test}}=3$.
On a common output grid $y\in[0,2]$, we compute,  the posterior large-deviation rate function  and the posterior quadratic rate induced by Gaussian-process regression with the fixed
NNGP kernel. In addition, for each output value $y$ we compute the relative kernel gap similarly as in Experiment 02A. %\emph{Output.}

The results are shown in Figure~\ref{fig:exp02B_post_ldp_vs_nngp}.
The left panel displays the posterior LDP rate function together with the posterior quadratic
rate induced by the NNGP kernel, while the right panel shows the corresponding relative kernel
gap. Unlike the prior case in Experiment~02A, the kernel gap is non-zero throughout the output
domain. This indicates that posterior concentration selects kernel configurations that differ from
the NNGP kernel at all output values, including near the posterior mode.
As a result, the posterior large-deviation geometry cannot be captured by a fixed-kernel
Gaussian approximation.
This behavior reflects genuine feature-learning effects that are absent in the NNGP theory. 
%even though the latter correctly describes typical fluctuations.

\begin{figure}[t]
\centering
\begin{minipage}{0.48\linewidth}
\centering
\includegraphics[width=\linewidth]{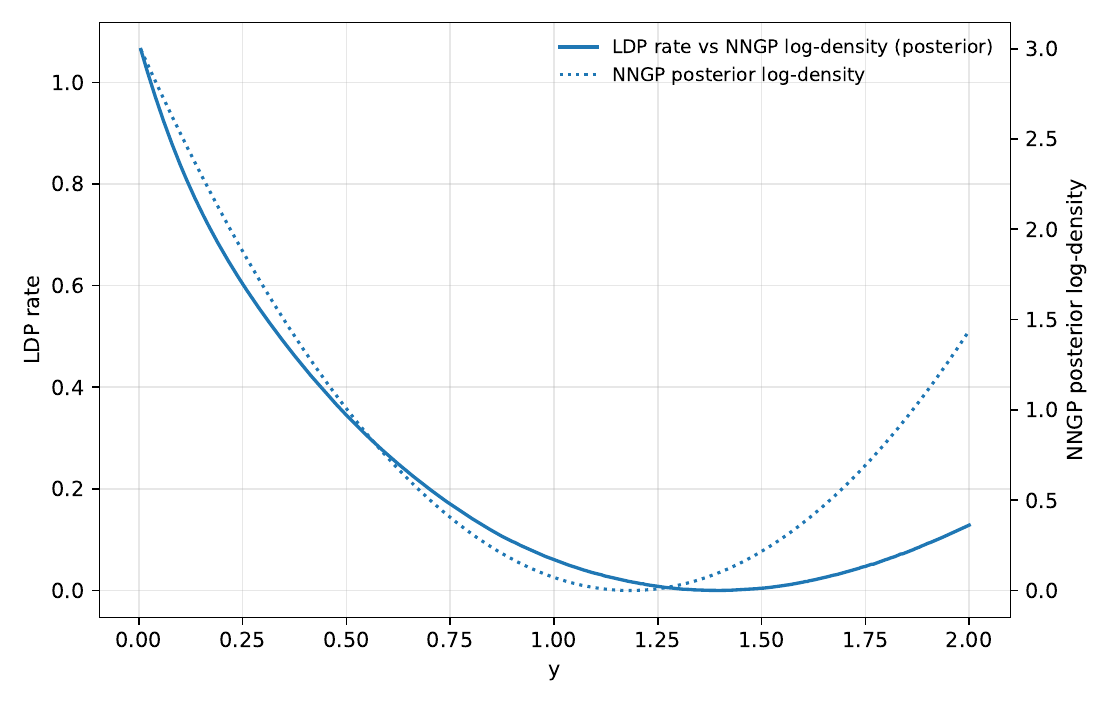}
\end{minipage}\hfill
\begin{minipage}{0.48\linewidth}
\centering
\includegraphics[width=\linewidth]{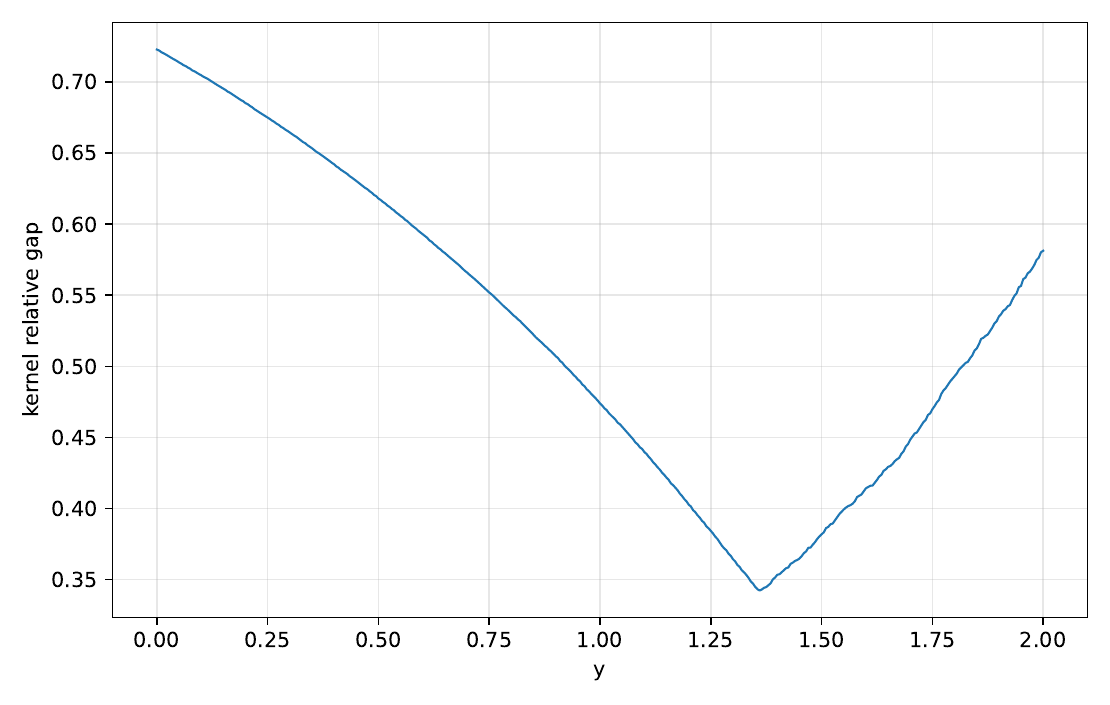}
\end{minipage}
\caption{\textbf{Posterior LDP versus NNGP.}
Left: posterior large-deviation rate function compared with the quadratic posterior rate
induced by Gaussian-process regression with the NNGP kernel.
Right: relative operator-norm gap between the kernel selected by the posterior LDP variational
problem and the NNGP kernel.}
\label{fig:exp02B_post_ldp_vs_nngp}
\end{figure}

\paragraph{02C: Predictive curve --- LDP-MAP versus NNGP posterior mean}

We compare the predictive behavior induced by the large-deviation posterior with that of the
Gaussian-process approximation based on the fixed NNGP kernel.
Specifically, we contrast the LDP-MAP predictor with the posterior mean of the NNGP
regression model. We consider the same shallow (ReLU) Gaussian neural network architecture, training set, and loss
function as in Experiments~01C and~02B, restricting again to the ReLU activation. For each test input $x_{\mathrm{test}} \in [-4, 4]$, we compute the large-deviation MAP prediction and the posterior mean prediction induced by Gaussian-process regression with the fixed NNGP kernel. In addition, along the same grid of test inputs, we compute a kernel-gap diagnostic
quantifying the deviation between the kernel selected by the LDP posterior and the NNGP
kernel. 

The results are shown in Figure~\ref{fig:exp02C_pred_curve_ldp_vs_nngp}. 
The two predictors agree near regions where the posterior geometry remains close to the
Gaussian regime. However, systematic discrepancies appear as the test input moves away from the training
set, in correspondence with non-negligible kernel gaps.
This confirms that predictive differences between LDP and NNGP formulations are directly
associated with kernel-selection effects arising at the large-deviation scale, rather than
with finite-width noise.

\begin{figure}[t]
\centering
\begin{minipage}{0.48\linewidth}
\centering
\includegraphics[width=\linewidth]{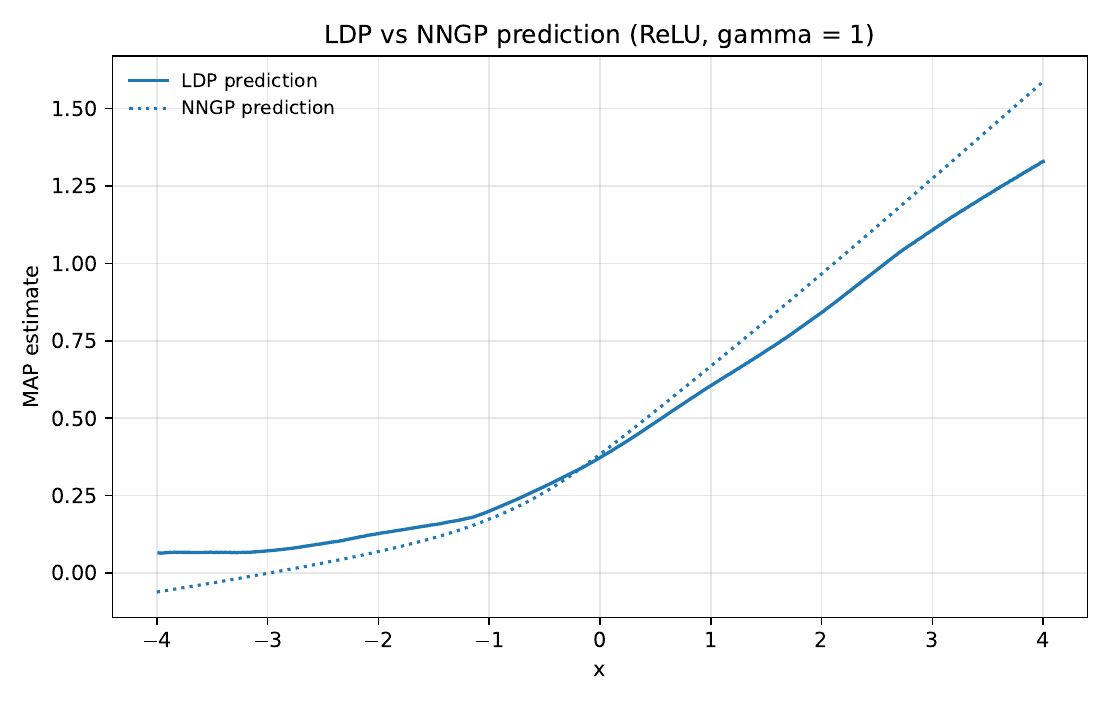}
\end{minipage}\hfill
\begin{minipage}{0.48\linewidth}
\centering
\includegraphics[width=\linewidth]{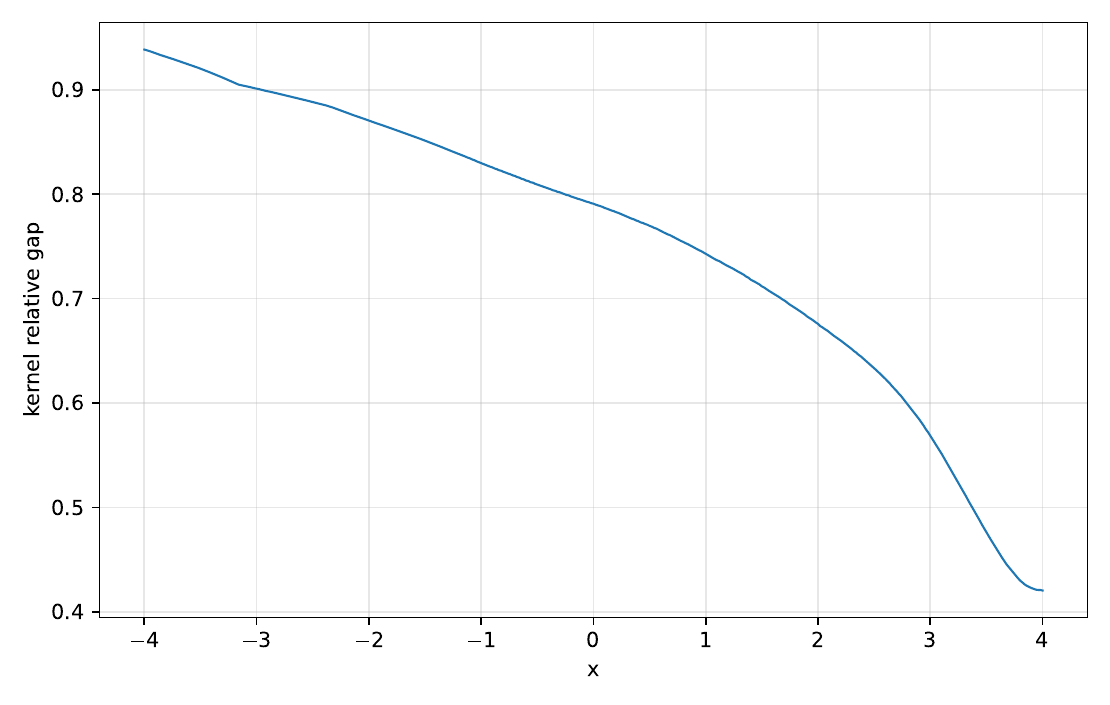}
\end{minipage}
\caption{\textbf{Predictive curves: LDP-MAP versus NNGP.}
Left: large-deviation MAP prediction compared with the NNGP posterior mean as a function of
the test input.
Right: kernel-gap diagnostic along the same input grid, quantifying deviations from the
fixed NNGP kernel.}
\label{fig:exp02C_pred_curve_ldp_vs_nngp}
\end{figure}

% --------------------------
\subsection{Experiment 03: finite-width validation via Monte Carlo}
\label{subsec:exp03}

We finally compare LDP predictions to Monte Carlo simulations of finite-width networks.
The focus is on validating large-deviation scaling for the prior and posterior
concentration 

\paragraph{03A: prior tails --- empirical decay vs LDP rate}

We verify that the LDP prior rate function governs the decay of empirical tail probabilities
as width increases. We sample from finite-width randomly LeCun initialized networks at widths $n\in\{32,64,128,256\}$ and empirically estimate the tail probabilities $\hat p_n(y)=\PP( \tfrac{1}{\sqrt{n}} h_\theta^n(x_{\mathrm{test}}) \ge y)$ (for positive $y$, and $\le y$ for negative side). We then compare empirical rates $-\tfrac 1 n\log \hat p_n(y)$ with $I_{\mathrm{prior}}(y)$.
The total number of Monte Carlo samples is $10^7$ for each width, and when no samples are observed in the tail region for a given width and threshold, the corresponding empirical probability is not reported.  The results are shown in Figure~\ref{fig:exp03A_prior_mc_vs_ldp}, panel left. The left panel displays empirical tail rates for increasing widths together with the
theoretical LDP prior rate function. As the width increases, the empirical tail rates progressively align with the LDP
prediction, providing evidence that the prior output distribution obeys the large-deviation
scaling predicted by the theory.

\begin{figure}[t]
\centering
\begin{minipage}{0.48\linewidth}
\centering
\includegraphics[width=\linewidth]{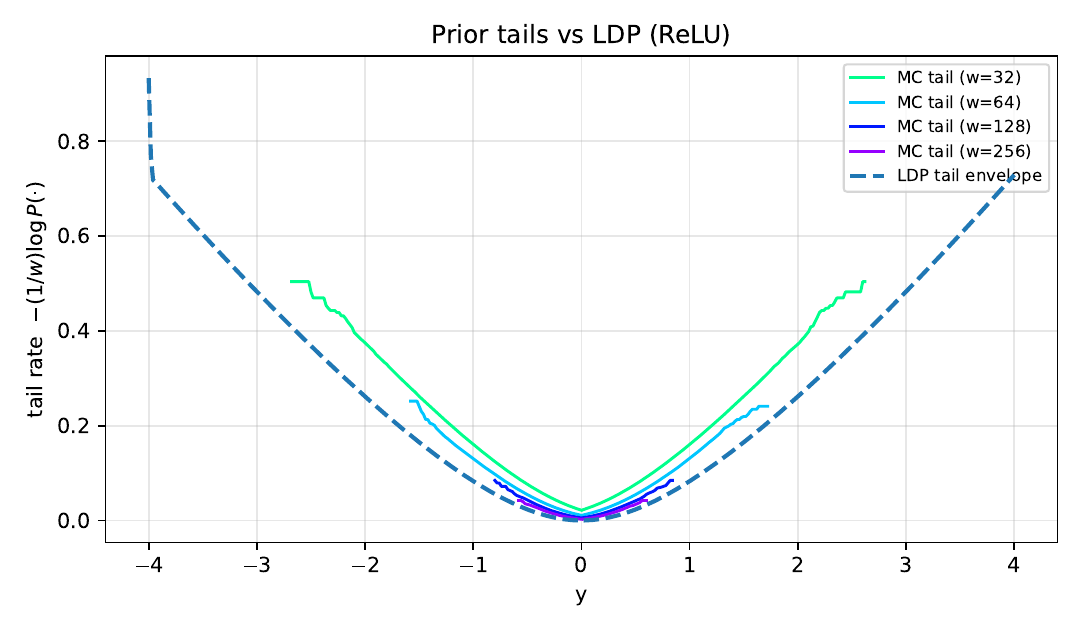}
\end{minipage}\hfill
\begin{minipage}{0.48\linewidth}
\centering
\includegraphics[width=\linewidth]{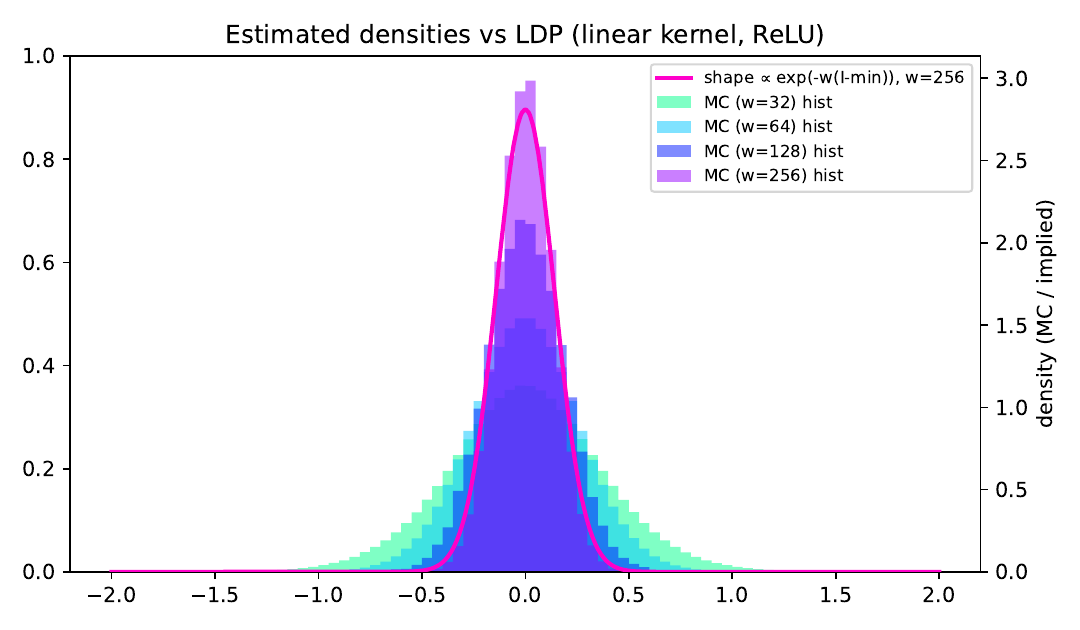}
\end{minipage}
\caption{\textbf{Prior tails: Monte Carlo versus LDP.}
Left: empirical tail decay rates $-\tfrac1n\log \hat p_n(y)$ for increasing widths compared
with the prior large-deviation rate function.
Right: empirical output histograms compared with the heuristic density
$\propto \exp(-n\,I_{\mathrm{prior}}(y))$.}
\label{fig:exp03A_prior_mc_vs_ldp}
\end{figure}

\paragraph{03B: Posterior samples --- MC vs LDP and NNGP.}
We compare finite-width posterior samples with the posterior geometry predicted by the
large-deviation principle and by the fixed-kernel NNGP approximation. We consider the same shallow Gaussian neural network with ReLU activation and the Heaviside training set as in Experiments~01B--02C.
Finite-width posterior samples are generated at width $n=128$ using a \emph{$n$-tempered likelihood}, so that the posterior concentrates at the large-deviation scale.
Sampling is performed using the Metropolis-adjusted Langevin algorithm (MALA), which we
found to provide more reliable convergence than SGLD in this low-dimensional setting.
Diagnostics and additional implementation details are reported in
Appendix~\ref{app:numerics}. Samples of the output $h_\theta(x_{\mathrm{test}})$ are collected at a fixed test input
$x_{\mathrm{test}}=5$, which is chosen to exhibit wider discrepancy between NNGP and LDP modes.

We compare two scalings: (i) the LDP scaling, for which the posterior rate function $\mathcal I_{\mathrm{post}}$
predicts concentration at speed $n$ of posterior samples of $\tfrac{1}{\sqrt{n}} h_\theta^n(x_{\mathrm{test}})$;  (ii) the standard NNGP scaling, corresponding to fixed-kernel Gaussian-process regression. We overlay the posterior LDP rate function (shown as an energy curve) and the NNGP posterior mean prediction.

Under LDP scaling, the empirical posterior exhibits strong concentration around a mode that
is accurately predicted by the LDP-MAP estimator.
Numerically, we observe mean$=1.936$, std$=0.187$, with an average MALA acceptance rate of $0.75$ across chains.
In contrast, under standard NNGP scaling the posterior distribution has mean$=1.879$, std$=1.977$, hence shifted and much broader (with acceptance rate $0.82$). The results are shown in Figure~\ref{fig:exp03B_post_mc_vs_ldp_nngp}.

\begin{figure}[t]
\centering
\vspace{-2em}
\begin{minipage}{0.48\linewidth}
\centering
\includegraphics[width=\linewidth]{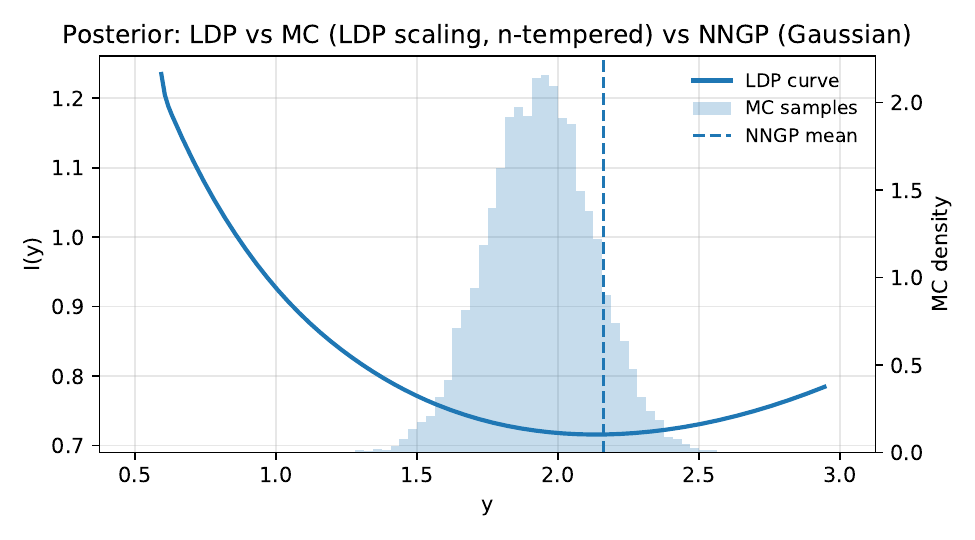}
\end{minipage}\hfill
\begin{minipage}{0.48\linewidth}
\centering
\includegraphics[width=\linewidth]{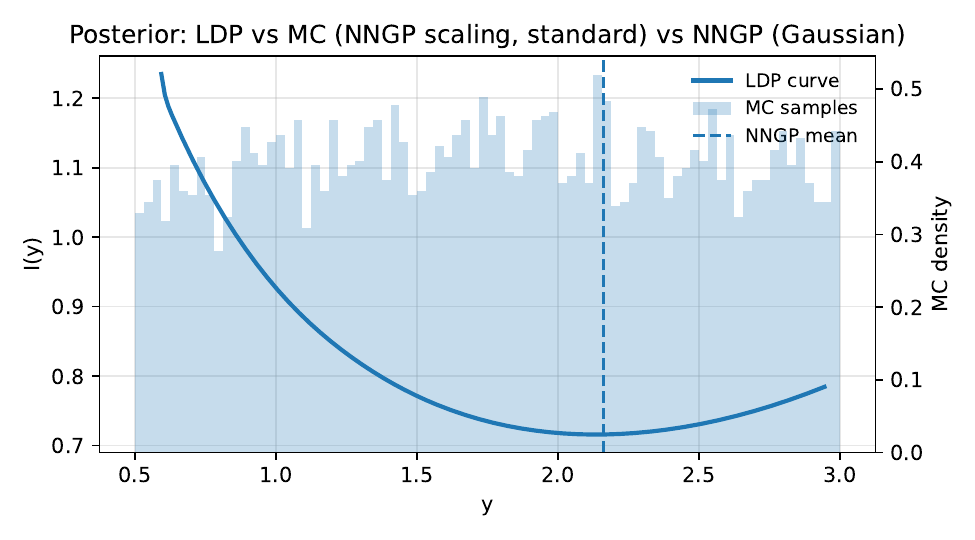}
\end{minipage}
\caption{\textbf{Posterior samples: MC vs LDP and NNGP.}
Finite-width posterior samples at $n=128$ for a ReLU network trained on a Heaviside target.
Left: $n$-tempered (LDP) scaling, showing concentration around the LDP-MAP prediction.
Right: standard NNGP scaling on the same $x$-axis, showing much broader Gaussian
fluctuations centered at the NNGP posterior mean.}
\label{fig:exp03B_post_mc_vs_ldp_nngp}
\end{figure}

\section{Conclusion}
We have shown that large-deviations provide a principled variational framework
for understanding Bayesian learning in wide neural networks beyond Gaussian-process limits.
Our numerics demonstrate that these effects are visible at moderate widths. Natural directions for future work include scaling the numerical framework to larger datasets,
systematic benchmarking against fixed-kernel and Bayesian neural network baselines, and
extending the theoretical analysis to more structured architectures.

\section{Acknowledgments}
D.T.\ thanks the HPC Italian National Centre for HPC, Big Data and Quantum Computing - Proposal code CN1 CN000\-00\-013, CUP I53C22000690001, the PRIN 2022 Italian grant 2022{-}WHZ5XH - ``understanding the LEarning process of QUantum Neural networks (LeQun)'', CUP J53D23003890006, the project  G24-202 ``Variational methods for geometric and optimal matching problems'' funded by Università Italo Francese.  D.T.\ and K.P.\ acknowledge the partial support of the project PNRR - M4C2 - Investimento 1.3, Partenariato Esteso PE00000013 - ``FAIR - Future Artificial Intelligence Research'' - Spoke 1 ``Human-centered AI'', funded by the European Commission under the NextGeneration EU programme, and  the MUR Excellence Department Project awarded to the Department of Mathematics, University of Pisa, CUP I57G22000700001,

%print bibilography with natbib file deep_learning.bib
\vfill

\bibliographystyle{plainnat}
%\nocite{*}
\bibliography{deep_learning}

\appendix

\section{A short guide to large-deviation calculus}
\label{app:ld}

Large deviations theory provides a systematic description of \emph{rare events}, namely
events whose probabilities decay exponentially with a scale parameter.
Throughout this appendix, we present an informal and calculus-oriented introduction,
mirroring familiar notions from elementary probability theory.
Standard references include \cite{dembo2009large,ellis2012entropy}.

\paragraph{Large deviation principle.}
Let $(Z_n)_{n\ge1}$ be random variables taking values in a fixed Euclidean space $\R^d$.
We say that $(Z_n)$ satisfies a \emph{large deviation principle} (LDP) with speed
$\beta_n \to \infty$ and rate function $\r:\R^d\to[0,\infty]$
(lower semicontinuous with compact level sets) if, for every Borel set $A$,
\begin{equation*}
\inf_{z\in A^\circ} \r(z)
\;\le\;
\liminf_{n\to\infty} -\frac1{\beta_n}\log \mathbb{P}(Z_n\in A)
\;\le\;
\limsup_{n\to\infty} -\frac1{\beta_n}\log \mathbb{P}(Z_n\in A)
\;\le\;
\inf_{z\in\overline A} \r(z),
\end{equation*}
where $A^\circ$ and $\overline A$ denote the interior and closure of $A$.
Heuristically, this is often summarized as
\[
\mathbb{P}(Z_n\in A)\;\approx\; \exp\!\bigl(-\beta_n\,\r(A)\bigr),
\qquad
\r(A):=\inf_{z\in A}\r(z).
\]

The rate function thus provides a ``skeleton'' of the probability law,
encoding which configurations are exponentially more likely than others.

\emph{Idempotent viewpoint.} A useful intuition is that large deviations obey the algebra of
\emph{idempotent (or tropical) probability}, where addition and multiplication
are replaced by $\min$ and $+$.
This perspective becomes transparent through the Laplace--Varadhan principle:
for every bounded continuous $\varphi:\R^d\to\R$,
\begin{equation}
\label{eq:varadhan}
\lim_{n\to\infty}
-\frac1{\beta_n}\log
\mathbb E\!\left[\exp\!\bigl(-\beta_n \varphi(Z_n)\bigr)\right]
=
\min_{z\in\R^d}\bigl\{\varphi(z)+\r(z)\bigr\}.
\end{equation}
The right-hand side is the idempotent analogue of an expectation.

Motivated by this analogy, it is convenient to introduce a formal
\emph{idempotent random variable} $Z$ and write $\r(Z=z)=\r(z)$.
We may then define  the idempotent expectation of $\varphi(Z)$ as follows
\[
\min_{z}\bigl\{\varphi(z)+\r(Z=z)\bigr\}.
\]

\begin{definition}[LDP convergence]
We say that $Z_n$ converges to $Z$ in the LDP sense with speed $\beta_n$,
and write $Z_n \stackrel{\beta_n}{\to} Z$, if for every bounded continuous $\varphi$,
\begin{equation}
\label{eq:ldp-conv}
\lim_{n\to\infty}
-\frac1{\beta_n}\log
\mathbb E\!\left[\exp\!\bigl(-\beta_n \varphi(Z_n)\bigr)\right]
=
\min_{z}\bigl\{\varphi(z)+\r(Z=z)\bigr\}.
\end{equation}
\end{definition}

This notion plays the role of convergence in law for rare-event asymptotics.

\paragraph{Basic calculus rules.}
Many standard probabilistic constructions admit direct idempotent analogues.

\begin{enumerate}
\item \emph{Normalization.}
Setting $\varphi\equiv0$ in \eqref{eq:ldp-conv} yields
$\min_z \r(z)=0$, the idempotent analogue of total mass one.
Accordingly, writing $\r(Z)\propto f$ means
\[
\r(Z=z)=f(z)-\min_u f(u),
\]
mirroring normalization of probability densities.

\item \emph{Change of measure.}
If $Z_n\stackrel{\beta_n}{\to} Z$ and we tilt the law of $Z_n$ by
$\exp(-\beta_n\ell(z))$, where $\ell$ is bounded and continuous, then
the LDP still holds with modified rate
\[
\tilde\r(Z=z)\propto \r(Z=z)+\ell(z).
\]
This is the LDP analogue of stability of weak convergence under bounded
Radon--Nikodym derivatives.

\item \emph{Contraction principle.}
If $Z_n\stackrel{\beta_n}{\to} Z$ and $g:\R^d\to\R^m$ is continuous, then
$g(Z_n)\stackrel{\beta_n}{\to} g(Z)$ with rate
\[
\r\bigl(g(Z)=y\bigr)
=
\min_{z:\,g(z)=y}\r(Z=z).
\]
This replaces marginalization by minimization.

\item \emph{Conditional LDP (Chaganty-type).}
If $X_n\stackrel{\beta_n}{\to}X$ and $Y_n$ satisfies an LDP conditionally on
$X_n=x_n$ for every $x_n\to x$, then $(X_n,Y_n)\stackrel{\beta_n}{\to}(X,Y)$ with
\[
\r(X=x,Y=y)=\r(X=x)+\r(Y=y\mid X=x).
\]

\item \emph{G\"artner--Ellis theorem.}
If the limit
\[
\hat\r(\lambda)
=
\lim_{n\to\infty}
\frac1{\beta_n}\log
\mathbb E\!\left[e^{\beta_n\langle\lambda,Z_n\rangle}\right]
\]
exists, is finite in a neighborhood of $0$, and is differentiable there,
then $Z_n\stackrel{\beta_n}{\to} Z$ with rate
\[
\r(Z=z)=\hat\r^{\,*}(z)
=
\sup_{\lambda\in\R^d}\bigl\{\langle\lambda,z\rangle-\hat\r(\lambda)\bigr\}.
\]
This is the LDP analogue of Lévy’s continuity theorem.

\item \emph{Cram\'er’s theorem.}
For empirical means $\bar X_n=\frac1n\sum_{i=1}^n X_i$ of i.i.d.\ variables, with MGF $\mathbb E\bigl[e^{\langle\lambda,X\rangle}\bigr]$ finite in a neighbourhood of $\lambda = 0$,
\[
\r(\bar X=x)
=
\sup_{\lambda}\Bigl\{\langle\lambda,x\rangle
-
\log\mathbb E\bigl[e^{\langle\lambda,X\rangle}\bigr]\Bigr\},
\]
the rare-event counterpart of the law of large numbers and central limit theorem.
\end{enumerate}

\vspace{0.5em}
Taken together, these rules form a compact ``large-deviation calculus''
in which probabilistic operations are replaced by variational ones.
This viewpoint underlies the variational formulations used throughout the paper.

% ============================================================
\section{Proofs of the main large-deviation statements}
\label{app:proof}

This appendix collects concise proofs (or proof sketches when the result is taken from the
literature) for the statements in the main body. Throughout we work on a fixed finite input
set $\cX$ (containing training and test inputs) and identify kernels with PSD matrices in
$\mathbb S_+^{|\cX|}$. We use the ``large-deviation calculus'' notation from
Appendix~\ref{app:ld}: for a sequence $(Z_n)_n$ we write $Z_n \stackrel{\beta_n}{\to} Z$ and
$\mathsf r(Z=z)$ for the associated rate function.

% ------------------------------------------------------------
\subsection{Prior LDP for layerwise kernels (sketch from \cite{macci2024large})}
\label{sec:proof-prior-kernel}

We recall the kernel recursion for Gaussian MLPs. For $\ell\ge 1$, define the empirical
kernel at layer $\ell$ on $\cX$ by
\[
K_n^{(\ell)}(\cX)
:= \frac1n\sum_{j=1}^n
\bigl(\sigma^{(\ell)}(h^{(\ell)}_j(\cX))\bigr)^{\otimes 2}
\in \mathbb S_+^{|\cX|},
\]
where $h^{(\ell)}_j(\cX)\in\R^{|\cX|}$ denotes the $j$-th neuron pre-activation vector
restricted to $\cX$ and $v^{\otimes 2}= v v^\top$. Under LeCun scaling, conditional on $K_n^{(\ell-1)}(\cX)=\kappa$,
the vectors $h^{(\ell)}_j(\cX)$ are i.i.d.\ Gaussian with covariance $\kappa$. In particular,
\begin{equation}\label{eq:cond-kernel-iid}
K_n^{(\ell)}(\cX)\;\Big|\; \{K_n^{(\ell-1)}(\cX)=\kappa\}
\;\stackrel{\text{law}}{=}\;
\frac1n\sum_{j=1}^n \, \Xi^{(\ell)}_j(\kappa),
\qquad
\Xi^{(\ell)}_j(\kappa)
:= \bigl(\sigma^{(\ell)}(\sqrt{\kappa}\,N_j)\bigr)^{\otimes 2},
\end{equation}
with $(N_j)_j$ i.i.d.\ $\cN(0,I_{|\cX|})$.

\emph{Conditional LDP and the layer cost.}
For fixed $\kappa$, \eqref{eq:cond-kernel-iid} is an empirical mean of i.i.d.\ matrices,
hence Cram\'er/G\"artner--Ellis yields an LDP at speed $n$ with conditional rate given by
the Legendre--Fenchel transform of the conditional log-MGF:
\begin{equation}\label{eq:def-Jsigma}
J^{\sigma^{(\ell)}}(\kappa' \mid \kappa)
:=
\sup_{\Lambda\in\mathbb S^{|\cX|}}
\Big\{
\langle \Lambda,\kappa'\rangle
-
\log\mathbb E\big[\exp\langle \Lambda, \Xi^{(\ell)}_1(\kappa)\rangle\big]
\Big\},
\qquad \kappa'\in\mathbb S_+^{|\cX|}.
\end{equation}

\emph{Induction and joint LDP.}
Assume inductively that $K_n^{(\ell-1)}(\cX)\stackrel{n}{\to}K^{(\ell-1)}(\cX)$ with
rate $I^{(\ell-1)}$. Under mild regularity, the conditional LDP in $\kappa$ above is
\emph{continuous} in the sense of Chaganty’s theorem (Appendix~\ref{app:ld}), and yields a
joint LDP for $\bigl(K_n^{(\ell-1)}(\cX),K_n^{(\ell)}(\cX)\bigr)$ with joint rate
$I^{(\ell-1)}(\kappa)+J^{\sigma^{(\ell)}}(\kappa'\mid \kappa)$.
By contraction onto the second coordinate we obtain the recursion
\begin{equation}\label{eq:kernel-ldp-recursion-proof}
I^{(\ell)}(\kappa')
=
\inf_{\kappa\in\mathbb S_+^{|\cX|}}
\Big\{ I^{(\ell-1)}(\kappa) + J^{\sigma^{(\ell)}}(\kappa'\mid \kappa)\Big\},
\end{equation}
which is \eqref{eq:kernel-ldp-recursion}. This argument (and the required regularity
conditions) is carried out in detail in \cite{macci2024large}, see also \cite{andreis2025ldp}.
In particular, the unique minimizer is the NNGP kernel $\kappa_0^{(\ell)}$, characterized
by the law-of-large-numbers fixed-point recursion.

% ------------------------------------------------------------
\subsection{Prior LDP for outputs (sketch from \cite{macci2024large})}
\label{sec:proof-prior-output}

Let
\[
H_n := \frac{1}{\sqrt n}\bigl(h_\theta(x)\bigr)_{x\in\cX}\in(\R^{d_{\mathrm{out}}})^{|\cX|}.
\]
Conditional on $K_n^{(L-1)}(\cX)=\kappa$, the output is Gaussian with covariance $\kappa$
(componentwise in $\R^{d_{\mathrm{out}}}$). Consequently, conditional large deviations of
$H_n$ have the quadratic rate
\begin{equation}\label{eq:cond-output-rate}
\mathsf r\bigl(H = h \,\big|\, K^{(L-1)}=\kappa \bigr)
=
\frac12\|h\|_\kappa^2,
\end{equation}
where $\|\cdot\|_\kappa$ denotes the (finite-dimensional) RKHS norm induced by $\kappa$ on
$\cX$, applied componentwise.

Combining \eqref{eq:cond-output-rate} with the kernel LDP at layer $L-1$ and applying
Chaganty's theorem plus contraction yields the unconditional output rate
\begin{equation}\label{eq:prior-output-minkappa-proof}
I_{\mathrm{prior}}^{(L)}(h)
=
\inf_{\kappa\in\mathbb S_+^{|\cX|}}
\Big\{ \tfrac12\|h\|_\kappa^2 + I^{(L-1)}(\kappa)\Big\},
\end{equation}
which is \eqref{eq:output-rate-min-kappa}. A complete proof appears in \cite{macci2024large}.

% ------------------------------------------------------------
\subsection{Posterior LDP under tempered quadratic loss}
\label{sec:proof-posterior}

We now prove Theorem~\ref{thm:posterior-output-ldp} from Theorem~\ref{thm:prior-output-ldp}
by a change-of-measure argument. %This is the functional counterpart of the standard ``Bayes rule for LDP'' stated in Appendix~\ref{app:ld}.

\begin{lemma}[Change of measure / Bayes rule for LDP]\label{lem:bayes-proof}
Let $Z_n$ be random variables taking values in $\R^d$ for some $d \ge 1$, and assume
$Z_n\stackrel{\beta_n}{\to} Z$ with  rate function $\mathsf r(Z=\cdot)$. Let
$\ell:\R^d\to\R$ be bounded and continuous and define the tilted laws
\[
\frac{d\widetilde{\PP}_n}{d\PP_n}(z)
=
\frac{\exp\{-\beta_n \ell(z)\}}{\EE_{\PP_n}[\exp\{-\beta_n \ell(Z_n)\}]}\,.
\]
Then under $\widetilde{\PP}_n$, $Z_n$ satisfies an LDP at speed $\beta_n$ with rate
\[
\widetilde{\mathsf r}(Z=z)
=
\ell(z)+\mathsf r(Z=z)-\inf_{u}\bigl\{\ell(u)+\mathsf r(Z=u)\bigr\}.
\]
\end{lemma}

\begin{proof}
For any bounded continuous $\varphi$ on $\R^d$,
\[
-\frac1{\beta_n}\log \widetilde{\EE}_n\!\left[e^{-\beta_n\varphi(Z_n)}\right]
=
-\frac1{\beta_n}\log
\frac{\EE_n\!\left[e^{-\beta_n(\varphi+\ell)(Z_n)}\right]}
{\EE_n\!\left[e^{-\beta_n\ell(Z_n)}\right]}.
\]
Apply Varadhan’s lemma to the numerator and denominator and subtract the limits, yielding
\[
\inf_z\{\varphi(z)+\ell(z)+\mathsf r(z)\}-\inf_z\{\ell(z)+\mathsf r(z)\}.
\]
This is exactly the LDP in the form of Varadhan-Laplace principle with rate $\widetilde{\mathsf r}$.
\end{proof}

\begin{proof}[Proof of Theorem~\ref{thm:posterior-output-ldp}]
Apply Lemma~\ref{lem:bayes-proof} to $Z_n=H_n$ with speed $\beta_n=n$ and
\[
\ell(h)=\mathcal L(h):=\frac12\sum_{i\in D}\|h(x_i)-y_i\|_2^2,
\]
corresponding to the tempered posterior density proportional to
$\exp\{-n\,\mathcal L(H_n)\}$ with respect to the prior law of $H_n$.
Since $H_n\stackrel{n}{\to} H$ under the prior with rate $I_{\mathrm{prior}}^{(L)}$ by
Theorem~\ref{thm:prior-output-ldp}, Lemma~\ref{lem:bayes-proof} gives the posterior rate
\[
I_{\mathrm{post}}^{(L)}(h)
=
\mathcal L(h)+I_{\mathrm{prior}}^{(L)}(h)-\inf_u\{\mathcal L(u)+I_{\mathrm{prior}}^{(L)}(u)\}.
\]
Substituting \eqref{eq:output-rate-min-kappa} into the right-hand side yields
\[
I_{\mathrm{post}}^{(L)}(h)
=
\inf_{\kappa\in\mathbb S_+^{|\cX|}}
\left\{
\frac12\sum_{i\in D}\|h(x_i)-y_i\|_2^2 + \frac12\|h\|_\kappa^2 + I^{(L-1)}(\kappa)
\right\}
\;+\;(\text{additive constant}),
\]
which is \eqref{eq:posterior-rate-min-kappa}.
\end{proof}

\begin{remark}
Although we proved Theorem~\ref{thm:posterior-output-ldp} only for bounded loss functions, the quadratic case follows along the same lines, from a refined version of Varadhan-Laplace principle,  \cite[Theorem~4.3.1]{dembo2009large}, which allows for unbounded continuous functions, provided an exponential tail condition is satisfied, which holds true in the quadratic (and more generally non-negative) cases. %$ \phi=-\ell\leq 0$, so
% \[
% e^{\phi/\varepsilon} \leq 1,
% \]
% and the tail condition (4.3.2) is satisfied directly. 
Hence, the posterior rate function is obtained by adding the loss to the prior rate function, up to normalization exactly  as in (\ref{eq:posterior-rate-min-kappa}).
\end{remark}

% ------------------------------------------------------------
\subsection{Kernel selection and separation from NNGP}
\label{sec:proof-kernel-separation}

We now prove Corollary~\ref{thm:kernel-separation}. For clarity we present the argument for
$d_{\mathrm{out}}=1$; the matrix-valued case follows by replacing quadratic forms by traces.

\begin{proof}[Proof of Corollary~\ref{thm:kernel-separation}]
Fix $\kappa\in\mathbb S_+^{|\cX|}$ and consider the inner minimization problem appearing in
\eqref{eq:posterior-rate-min-kappa}:
\[
\inf_{h:\cX\to\R}\left\{
\frac12\sum_{i\in D}(h(x_i)-y_i)^2+\frac12\|h\|_\kappa^2
\right\}.
\]
By the representer theorem on the finite set $\cX$, the minimizer has the form
$h(\cdot)=\sum_{j\in D}\alpha_j \kappa(\cdot,x_j)$, i.e. $h=\kappa_{\cdot D}\alpha$.
Writing $K:=\kappa_{DD}$, $k_x:=\kappa_{xD}$, we have
\[
\|h\|_\kappa^2=\alpha^\top K\,\alpha,
\qquad
h_D = K\alpha.
\]
Therefore the objective becomes, as a function of $\alpha$,
\[
\frac12\|K\alpha-y_D\|_2^2+\frac12\,\alpha^\top K\alpha.
\]
This is strictly convex on $\mathrm{Ran}(K)$ and its minimizer satisfies
\[
(K+\mathrm{Id})K\alpha = K y_D
\quad\Longrightarrow\quad
(K+\mathrm{Id})\alpha = y_D \ \ \text{on }\mathrm{Ran}(K),
\]
hence we may take $\alpha=(K+\mathrm{Id})^{-1}y_D$ (interpreting the inverse as the
Moore--Penrose inverse if needed). Plugging back yields the minimal value
\[
\inf_h\left\{\frac12\sum_{i\in D}(h(x_i)-y_i)^2+\frac12\|h\|_\kappa^2\right\}
=
\frac12\,y_D^\top (K+\mathrm{Id})^{-1}y_D.
\]
Consequently, after minimizing out $h$, the $\kappa$-dependent posterior objective becomes
\[
I^{(L-1)}(\kappa)+\frac12\,y_D^\top(\kappa_{DD}+\mathrm{Id})^{-1}y_D,
\]
which is exactly \eqref{eq:phi-kappa} (up to an additive constant independent of $\kappa$).

Finally, let $\kappa_0$ be the unique minimizer of $I^{(L-1)}$ (the NNGP kernel).
If $y_D\neq 0$, the map $\kappa\mapsto \frac12\,y_D^\top(\kappa_{DD}+\mathrm{Id})^{-1}y_D$
is strictly \emph{decreasing} along positive semidefinite directions, hence it is not
constant in any neighborhood of $\kappa_0$. Therefore $\kappa_0$ cannot minimize the sum
unless the data-term is identically constant, which only happens for $y_D=0$.
Thus any posterior-optimal kernel $\kappa^\star$ satisfies $\kappa^\star\neq\kappa_0$.
\end{proof}

The computation above is the precise form of the ``min--min exchange'' discussed in the
main text: the posterior selects $\kappa$ by minimizing a \emph{kernel-level} objective
obtained after eliminating $h$ (equivalently, the dual variables $\alpha$). This makes
explicit that, unlike GP regression with fixed kernel, the LDP posterior performs an
implicit \emph{kernel selection} step.

\section{Linear activation: explicit large-deviation formulas}
\label{app:explicit}

This appendix illustrates the large-deviation framework in the simplest analytically
tractable setting: linear activation functions and a single input.
While overly simplistic from a modeling perspective, this case provides explicit formulas
that clarify the structure of kernel and output rate functions and explain the non-quadratic
geometries observed numerically in Section~\ref{sec:numerics}.

Throughout, we consider the linear activation
\[
\sigma(x) = \sqrt{a}\,x , \qquad a>0,
\]
and adopt the notation and layer indexing of the main text.

% ------------------------------------------------------------
\paragraph{Prior rate function for layer-wise kernels}

Let $K^{(\ell)}$ denote the random kernel at layer $\ell$, and let
\[
I^{(\ell)}(\kappa) := \mathsf{r}\bigl(K^{(\ell)}=\kappa\bigr)
\]
denote its prior large-deviation rate function.
We work in the single-input setting, so kernels reduce to non-negative scalars
$\kappa \in \mathbb{R}_+$.

For linear activation, the conditional log-moment generating function can be computed
explicitly, yielding the Legendre transform
\begin{equation}
\label{eq:linear_phi_star}
J(\kappa \mid \kappa_0)
=
\sup_{\lambda \in \mathbb{R}}
\left\{
\lambda \kappa + \frac{1}{2}\log\bigl(1-2\lambda a \kappa_0\bigr)
\right\}.
\end{equation}
Stationarity gives
\[
\kappa = \frac{a\kappa_0}{1-2\lambda a\kappa_0},
\]
and substitution yields the explicit formula
\begin{equation}
\label{eq:linear_phi_star_explicit}
J(\kappa \mid \kappa_0)
=
\frac12\left[
\frac{\kappa}{a\kappa_0}
-\log\!\left(\frac{\kappa}{a\kappa_0}\right)
-1
\right].
\end{equation}

\emph{Shallow case ($L=1$).}
For a single hidden layer, the kernel rate function coincides with $J*$:
\begin{equation}
\label{eq:I2_linear}
I^{(1)}(\kappa)
=
\frac12\left[
\frac{\kappa}{a\kappa^{(0)}}
-\log\!\left(\frac{\kappa}{a\kappa^{(0)}}\right)
-1
\right].
\end{equation}
This function is convex, with a unique minimum at
$\kappa = a\kappa^{(0)}$, corresponding to the NNGP kernel.

\emph{Deeper networks.}
For depth $L\ge2$, the kernel rate functions satisfy the recursive variational relation
\[
I^{(\ell)}(\kappa)
=
\inf_{\kappa'}
\Bigl\{
I^{(\ell-1)}(\kappa')
+
\Phi^*(\kappa \mid \kappa')
\Bigr\}.
\]
This recursion can be solved explicitly by induction.

\begin{theorem}[Explicit kernel rate functions for linear activation]
\label{thm:linear_kernel_rates}
Assume linear activation $\sigma(x)=\sqrt{a}\,x$ and a single input.
Then for every depth $L\ge1$ and $\kappa\ge0$,
\begin{equation}
\label{eq:linear_kernel_rate_general}
I^{(L)}(\kappa)
=
\frac{L}{2}
\left[
\left(\frac{\kappa}{a^{L}\kappa^{(0)}}\right)^{\!1/L}
-
\log\!\left(\frac{\kappa}{a^{L}\kappa^{(0)}}\right)^{\!1/L}
-1
\right].
\end{equation}
The function $I^{(L)}$ has a unique global minimum at
$\kappa = a^{L}\kappa^{(0)}$.
For $L>1$, it is not globally convex.
\end{theorem}

This loss of convexity reflects the increasing flexibility of kernel fluctuations as depth
grows and foreshadows the non-quadratic behavior of output rate functions.

% ------------------------------------------------------------
\paragraph{Prior rate function for the network output}

We now turn to the prior large-deviation rate function of the network output.
In the single-input, scalar-output setting, this reduces to a function
$I^{(L)}_{\mathrm{prior}}(y)$ with $y\in\mathbb{R}$.

\emph{Shallow case ($L=1$).}
Using \eqref{eq:I2_linear}, the output rate function is obtained by minimizing over
$\kappa\ge0$:
\[
I^{(1)}_{\mathrm{prior}}(y)
=
\inf_{\kappa\ge0}
\left\{
\frac12\left[
\frac{\kappa}{a\kappa^{(0)}}
-\log\!\left(\frac{\kappa}{a\kappa^{(0)}}\right)
-1
\right]
+
\frac{y^2}{2\kappa}
\right\}.
\]
The minimization  can be performed explicitly, yielding to the following expression:
\begin{align}
I^{(1)}_{\mathrm{prior}}(y)
&=
\frac14\Bigl(1+\sqrt{1+\tfrac{4y^2}{a\kappa^{(0)}}}\Bigr)
-\frac12\log\!\Bigl[\tfrac12\Bigl(1+\sqrt{1+\tfrac{4y^2}{a\kappa^{(0)}}}\Bigr)\Bigr]
-\frac12
\notag\\
&\quad
+
\frac{y^2}{a\kappa^{(0)}\bigl(1+\sqrt{1+\tfrac{4y^2}{a\kappa^{(0)}}}\bigr)} .
\label{eq:linear_output_rate}
\end{align}
Actually, still in the single-input setting, the case of multi-dimensional outputs $y \in \mathbb{R}^{d_{\mathrm{out}}}$ can be handled similarly, leading to radial rate functions, $I^{(1)}_{\mathrm{prior}}(y)$ (just replace $y^2$ with the squared norm $|y|^2$) and satisfying
\[
I^{(1)}_{\mathrm{prior}}(y)\sim |y|
\qquad\text{as } |y|\to\infty .
\]
This linear growth should be compared also the tails observed numerically for ReLU activations.

\paragraph{Deep case (\(L>1\)).}
For \(L>1\), an explicit formula is no longer available, but the asymptotic growth of
\(I_{\mathrm{prior}}^{(L)}\) can be characterized.

Let \(\kappa^*(y) = \kappa^{(L),*}(y)\) denote the minimizer in the definition of $I_{\mathrm{prior}}^{(L)}(y)$. The first-order optimality condition implies that \(\kappa^\star(y)\) satisfies
\begin{equation}
\label{eq:kappa_star_eq}
\frac{1}{(a^{L}\kappa^{(0)})^{1/L}}
\bigl(\kappa^\star\bigr)^{\frac{L+1}{L}}
-
\kappa^\star
-
y^2
=0.
\end{equation}
This relation can be rewritten as
\begin{equation}
\label{eq:kappa_star_scaling}
\kappa^\star
=
\bigl(a^{L}\kappa^{(0)}\bigr)^{\frac{1}{L+1}}
\bigl(y^2+\kappa^\star\bigr)^{\frac{L}{L+1}} .
\end{equation}

To extract the asymptotic behavior, assume that
\(
\kappa^\star(y)\sim |y|^{d}
\)
for some exponent \(d>0\).
If \(d\ge2\), then the right-hand side of \eqref{eq:kappa_star_scaling} scales as
\(|y|^{dL/(L+1)}\), which is inconsistent since \(d\neq dL/(L+1)\) for \(L\ge1\).
Hence \(d<2\), and the dominant contribution comes from the \(y^2\) term, yielding
\[
d = \frac{2L}{L+1}.
\]
Substituting this scaling into the variational definition of $I^{(L)}_{\mathrm{prior}}$ gives
\begin{equation}
\label{eq:linear_output_asymptotic}
I_{\mathrm{prior}}^{(L)}(y)
\;\asymp\;
|y|^{\frac{2}{L+1}},
\qquad |y|\to\infty.
\end{equation}

Thus, the prior output rate function exhibits \emph{sublinear growth}, with an exponent that
decreases as depth increases.
This consideration may suggest why deeper networks with ReLU activation exhibit increasingly heavy-tailed
large-deviation behavior (even in the absence of nonlinear feature learning).

\section{Numerical diagnostics and optimization stability}
\label{app:numerics}

This appendix collects numerical diagnostics supporting the stability and reliability of the
computations reported in Section~\ref{sec:numerics}.
All experiments are implemented in \texttt{Python/JAX}.
While the large-deviation formulations lead to finite-dimensional convex or nearly convex
optimization problems, some care is required to handle numerical conditioning and curvature,
especially when optimizing over kernel variables.
The figures in this appendix illustrate that these issues are well controlled in practice.

\paragraph{Optimization of prior and posterior rate functions.}
Figures~\ref{fig:diag_prior}, \ref{fig:diag_post}~and \ref{fig:diag_mode} report diagnostics for the computation of
the prior, posterior output rate functions, and prediction respectively, in the ReLU case and Heaviside training dataset.
Each figure consists of four panels.

\emph{Top panel:} the prior/posterior rate function evaluated on the output grid.

\emph{Second panel:}  In orange, the norm of the gradient during the final iterations of the
\emph{inner optimization}, corresponding to the Legendre--Fenchel transform
$J^{\sigma}(\kappa \mid \kappa_0)$ of the conditional $\log$-MGF in \eqref{eq:kernel-ldp-recursion}. This optimization is performed using a hybrid Adam--L-BFGS scheme, with Adam providing robust initial exploration and L-BFGS ensuring fast local convergence. In blue, the norm of the gradient during the \emph{outer optimization} over the kernel variable $\kappa$, carried out using Adam only.

\emph{Third panel:} The kernel is parametrized via its Cholesky factorization to enforce positive semidefiniteness.
The plot reports the minimum diagonal entry of the optimized kernel, confirming that the
solution remains well inside the positive definite cone and away from numerical degeneracy.

\emph{Bottom panel:} the relative operator-norm gap between the optimized kernel $\kappa^\star$
and the NNGP kernel $\kappa_0$.
As discussed in the main text, this gap vanishes at the typical point for the prior and remains
nonzero under posterior conditioning, reflecting kernel selection effects.

Together, these diagnostics indicate stable convergence of both the inner and outer optimization
loops, with no signs of numerical pathologies.

\begin{figure}[t]
\centering
\includegraphics[width=\linewidth]{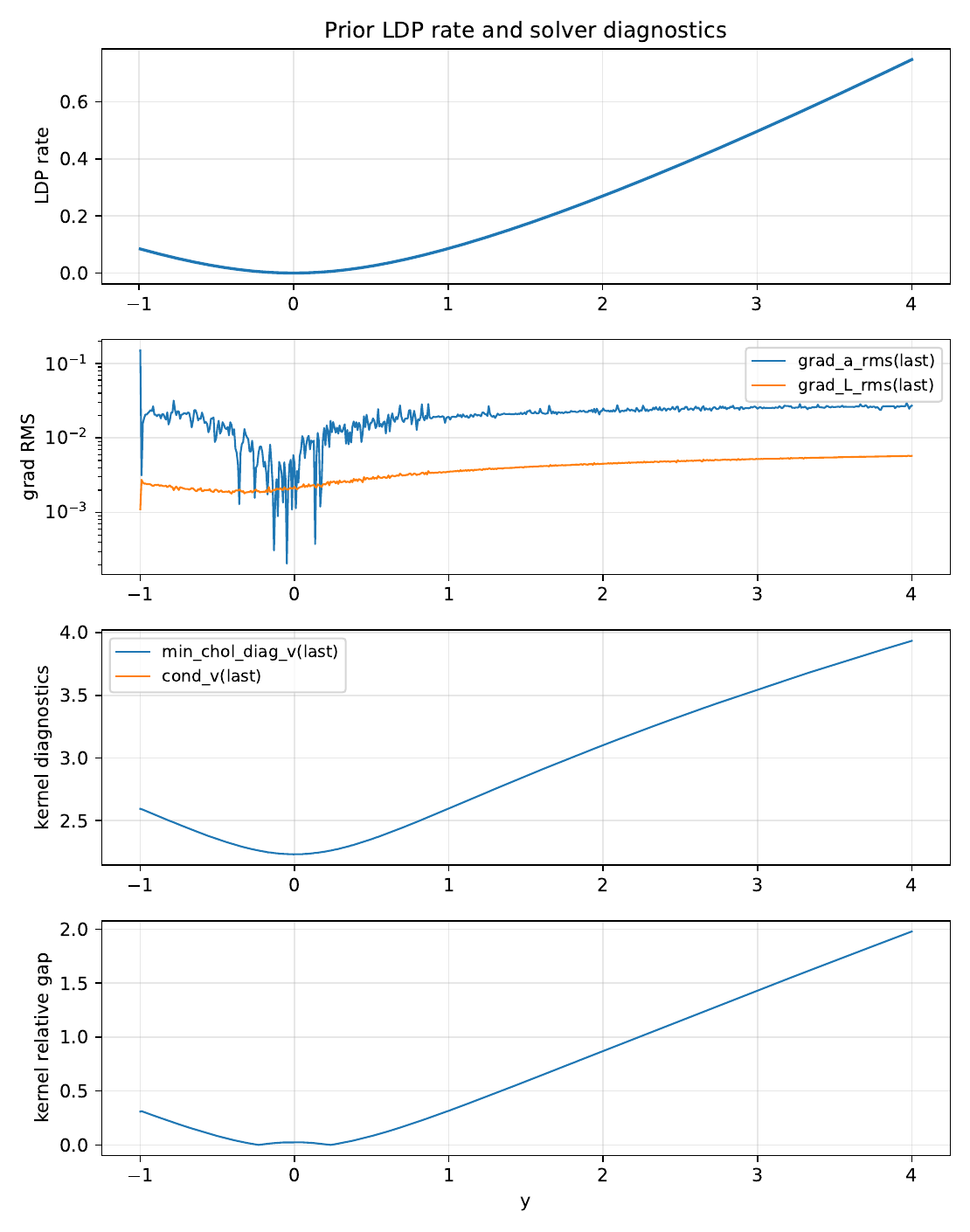}
\caption{\textbf{Diagnostics for prior rate computation (ReLU).}
Rate function, inner and outer gradient norms, minimum kernel diagonal entry, and kernel gap
relative to the NNGP kernel.}
\label{fig:diag_prior}
\end{figure}

\begin{figure}[t]
\centering
\includegraphics[width=\linewidth]{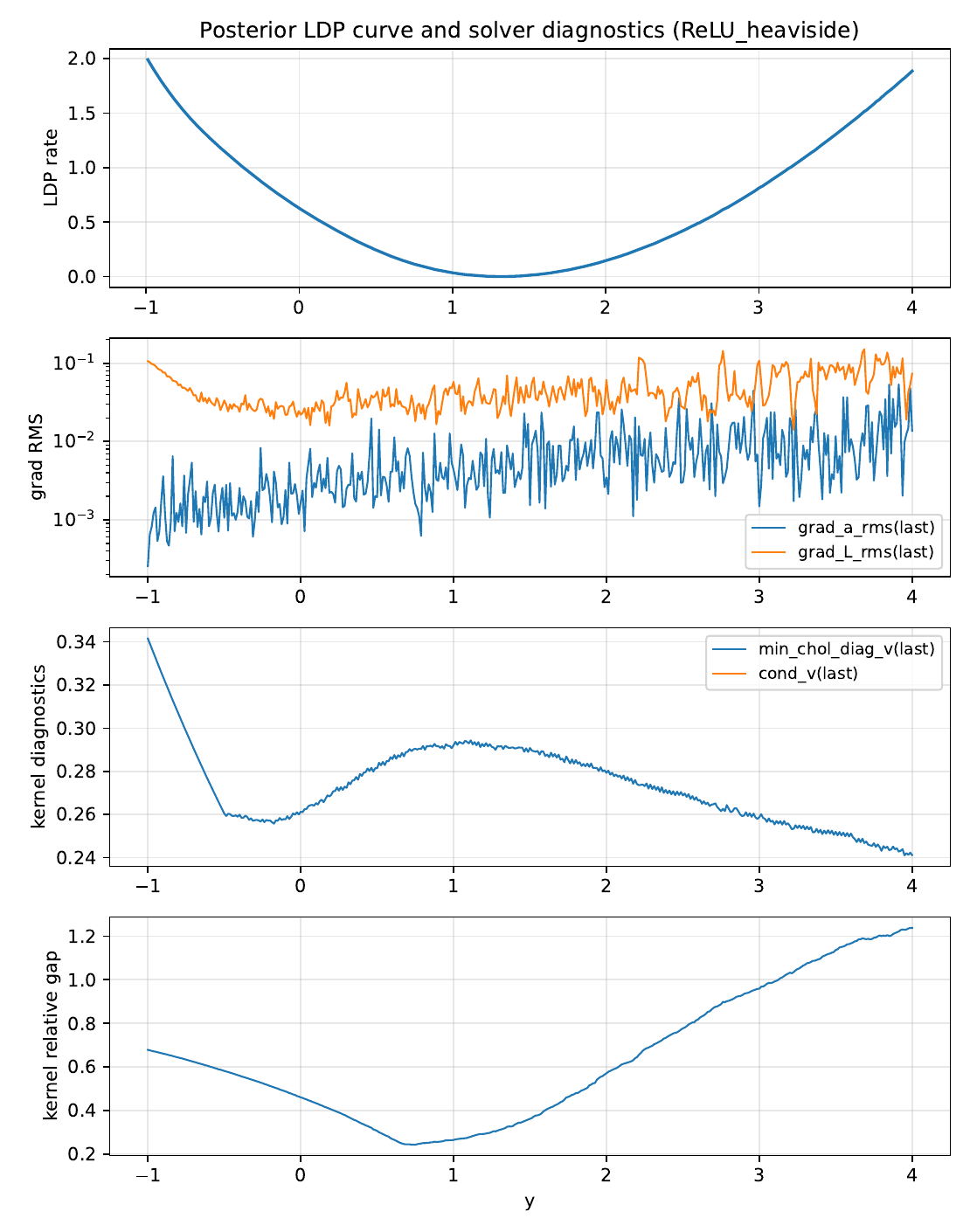}
\caption{\textbf{Diagnostics for posterior rate computation (ReLU, Heaviside data).}
Same layout as Figure~\ref{fig:diag_prior}, illustrating stable optimization under posterior
conditioning.}
\label{fig:diag_post}
\end{figure}

\begin{figure}[t]
\centering
\includegraphics[width=\linewidth]{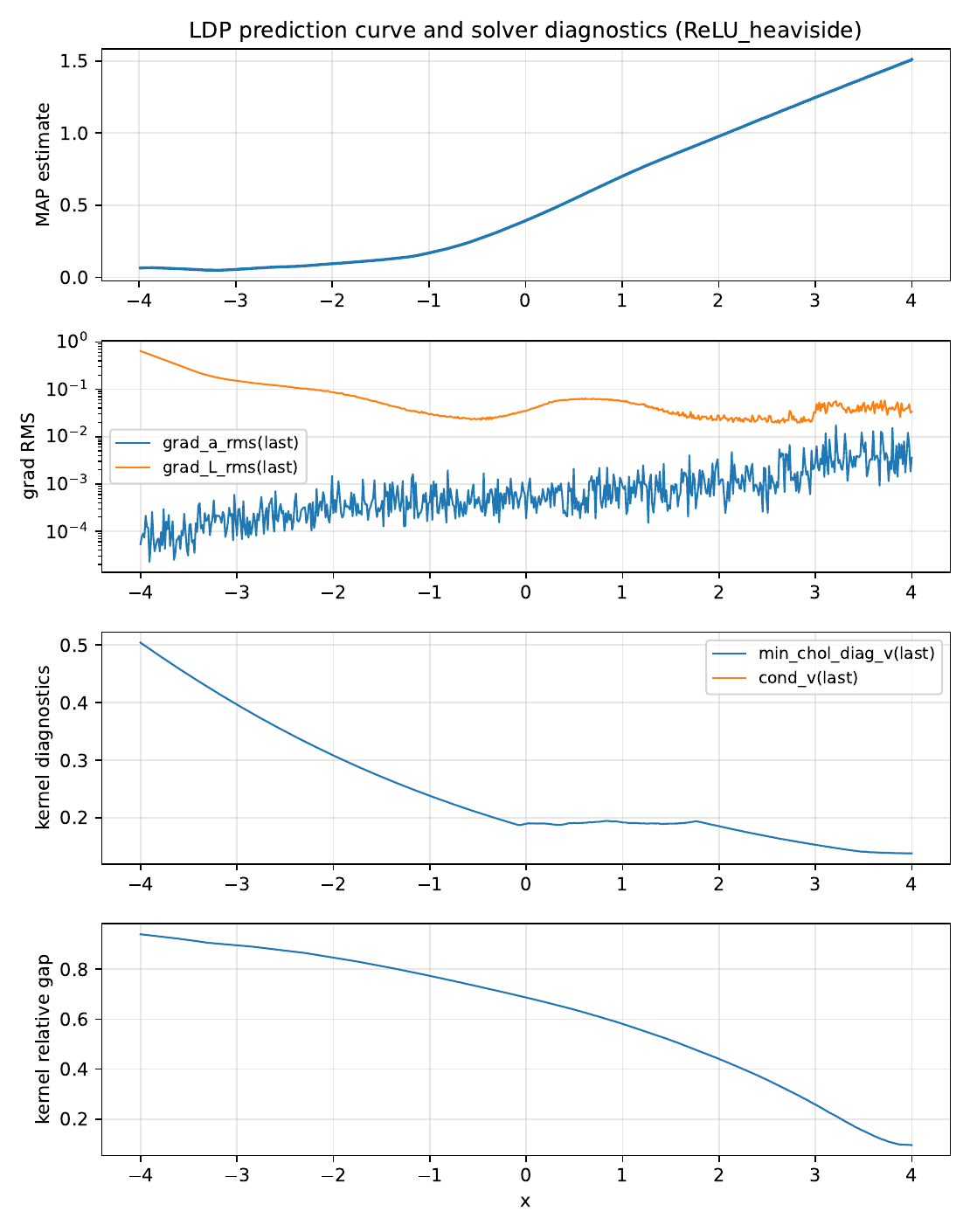}
\caption{\textbf{Diagnostics for MAP prediction computation (ReLU, Heaviside data).}
Same layout as Figure~\ref{fig:diag_prior}, illustrating stable optimization also for the prediction.}
\label{fig:diag_mode}
\end{figure}

\paragraph{Monte Carlo diagnostics for posterior sampling.}
Figure~\ref{fig:diag_mc} reports diagnostics for the Monte Carlo simulations used in
Experiment~03B.
Samples are drawn using the Metropolis-adjusted Langevin algorithm (MALA), which was preferred
over SGLD due to its superior stability in the presence of the strong curvature induced by the
$n$-tempered likelihood.

The figure shows trace plots for multiple (2 for clearer visibility -- while the main body experiment was performed on 10) independent chains targeting the finite-width
posterior at width $n=128$.
Despite the strong concentration induced by the large-deviation scaling, the chains mix
adequately and do not exhibit sticking or metastability.
Acceptance rates are reported in the main text and remain in a stable regime across chains.

These diagnostics support the claim that the empirical posterior distributions used for
comparison are reliably sampled and that discrepancies between Monte Carlo estimates and
NNGP predictions are not due to sampling artifacts.

\begin{figure}[t]
\centering
\includegraphics[width=\linewidth]{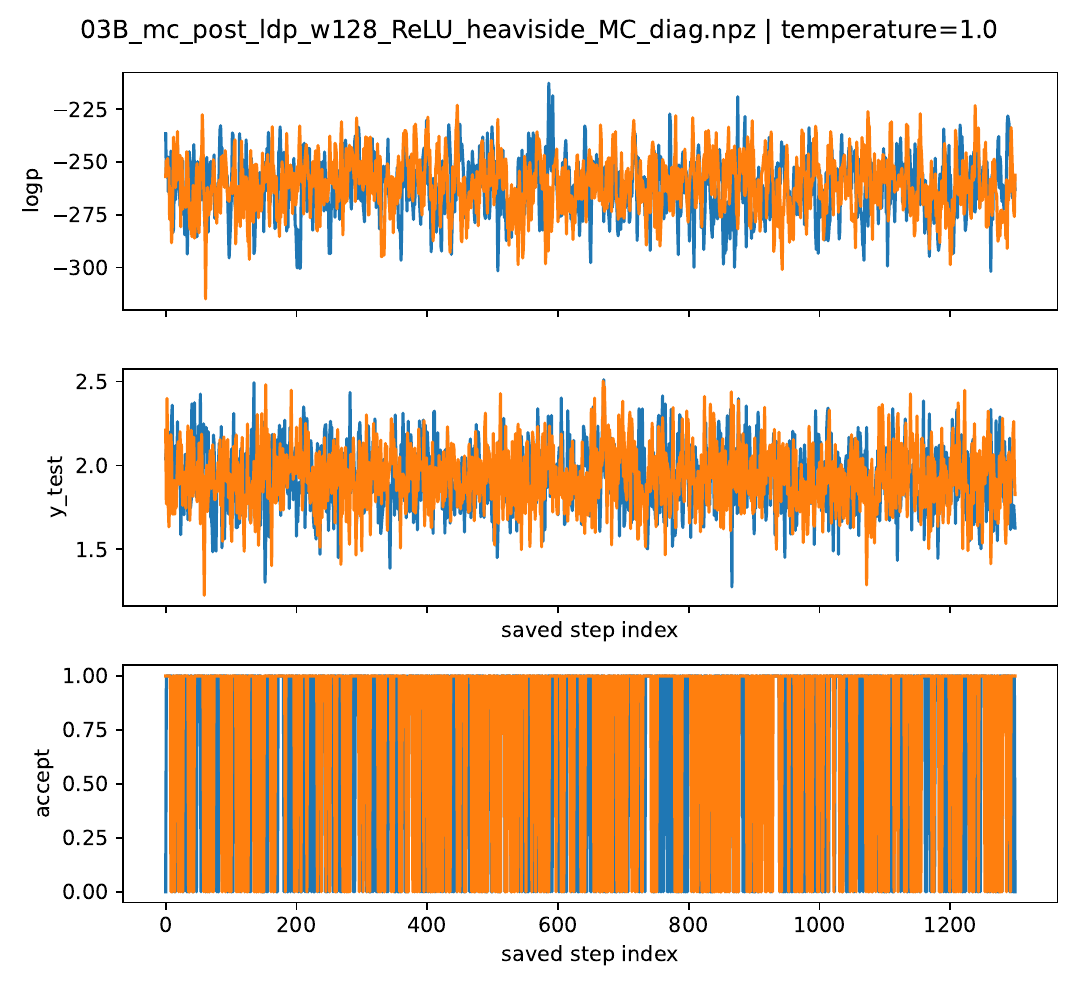}
\caption{\textbf{MCMC diagnostics for finite-width posterior sampling.}
Trace plots for multiple (2) MALA chains targeting the $n$-tempered posterior at width $n=128$.}
\label{fig:diag_mc}
\end{figure}

\end{document}